\documentclass{article}
\usepackage{microtype}
\usepackage{graphicx}
\usepackage{subcaption}
\usepackage{booktabs} 
\usepackage{hyperref,enumitem}

\newcommand{\Aut}{\mathrm{Aut}}
\newcommand{\gl}{\mathrm{GL}}
\newcommand{\IR}{\mathbb{R}}

\newcommand{\IC}{\mathbb{C}}
\newcommand{\spn}{\mathrm{span}}
\newcommand{\tr}{\mathrm{tr}}
\newcommand{\im}{\mathrm{im}}

\makeatletter
\newcommand*{\inlineequation}[2][]{%
  \begingroup
    \refstepcounter{equation}%
    \ifx\\#1\\%
    \else
      \label{#1}%
    \fi
    \relpenalty=10000 %
    \binoppenalty=10000 %
    \ensuremath{%
      #2%
    }%
    ~\@eqnnum
  \endgroup
}
\makeatother

\usepackage[preprint]{icml2026}

\usepackage{amsmath}
\usepackage{amssymb}
\usepackage{mathtools}
\usepackage{amsthm}
\usepackage{nicematrix}
\usepackage{tikz}
\usetikzlibrary{arrows,decorations.pathmorphing,backgrounds,positioning,fit,petri}

\usepackage[capitalize,compress]{cleveref}

\theoremstyle{plain}
\newtheorem{theorem}{Theorem}[section]

\newtheorem{lemma}[theorem]{Lemma}
\newtheorem{corollary}[theorem]{Corollary}
\newtheorem{example}[theorem]{Example}
\theoremstyle{definition}
\newtheorem{definition}[theorem]{Definition}

\theoremstyle{remark}

\usepackage[textsize=tiny]{todonotes}

\icmltitlerunning{Representation-Theoretic GNN Readouts}

\begin{document}

\twocolumn[
  \icmltitle{Breaking Symmetry Bottlenecks in GNN Readouts}

  \icmlsetsymbol{equal}{*}

  \begin{icmlauthorlist}
    \icmlauthor{Mouad Talhi}{xxx}
    \icmlauthor{Arne Wolf}{yyy,zzz}
    \icmlauthor{Anthea Monod}{yyy}
  \end{icmlauthorlist}

  \icmlaffiliation{xxx}{Department of Computing, Imperial College London, UK}
  \icmlaffiliation{yyy}{Department of Mathematics, Imperial College London, UK}
  \icmlaffiliation{zzz}{London School of Geometry and Number Theory, UK}

  \icmlcorrespondingauthor{Anthea Monod}{a.monod@imperial.ac.uk}

  \icmlkeywords{Machine Learning, ICML}

  \vskip 0.3in
]

\printAffiliationsAndNotice{} 

\begin{abstract}

Graph neural networks (GNNs) are widely used for learning on structured data, yet their ability to distinguish non-isomorphic graphs is fundamentally limited.  These limitations are usually attributed to message passing; in this work we show that an independent bottleneck arises at the readout stage.  Using finite-dimensional representation theory, we prove that all linear permutation-invariant readouts, including sum and mean pooling, factor through the Reynolds (group-averaging) operator and therefore project node embeddings onto the fixed subspace of the permutation action, erasing all non-trivial symmetry-aware components regardless of encoder expressivity.  This yields both a new expressivity barrier and an interpretable characterization of what global pooling preserves or destroys.  To overcome this collapse, we introduce projector-based invariant readouts that decompose node representations into symmetry-aware channels and summarize them with nonlinear invariant statistics, preserving permutation invariance while retaining information provably invisible to averaging.  Empirically, swapping only the readout enables fixed encoders to separate WL-hard graph pairs and improves performance across multiple benchmarks, demonstrating that readout design is a decisive and under-appreciated factor in GNN expressivity.

\end{abstract}

\section{Introduction}
Graph neural networks (GNNs) are a standard tool for learning on structured data, with applications ranging from chemistry and biology to social networks.  A large body of work studies their expressive power through the Weisfeiler--Leman (WL) hierarchy and message-passing limitations \cite{xu2019how,morris2019weisfeiler,grohe2021logic}, showing that many architectures fail to distinguish certain non-isomorphic graphs even with highly expressive node updates.

This literature focuses primarily on the encoder---the message-passing layers that produce node embeddings---while the readout, which compresses these embeddings into a graph-level representation, is typically treated as a minor design choice. This is problematic: the readout is the only stage where node-level information is aggregated, and a poor choice can irreversibly discard structure that was already computed.  In practice, most GNNs rely on simple permutation-invariant pooling such as sum or mean yet the representational effect of these operations remains poorly understood.

In this paper, we show that standard GNN readouts impose a fundamental and previously overlooked expressivity bottleneck that is independent of message passing.  We use representation theory to make this bottleneck explicit.

\textbf{Readout as a Representation-Theoretic Projection.}  We analyze readouts through the natural action of the node permutation group on node embeddings.  Under this action, sum and mean pooling factor through the classical \emph{Reynolds operator} from invariant theory, which projects any representation onto its fixed (permutation-invariant) subspace.  Using finite-dimensional representation theory, we show that this projection eliminates all non-trivial symmetry-aware components of the embeddings.  Consequently, no linear permutation-invariant readout can access information outside the trivial component, regardless of how expressive the encoder is.  We formalize this as a factorization theorem: every linear invariant readout necessarily factors through the Reynolds projection.  This yields a second expressivity bottleneck in GNNs, orthogonal to WL-type limitations, that arises purely at the readout level.

Beyond identifying this bottleneck, our analysis yields an explicit form of model interpretability: it characterizes exactly which symmetry components of the learned node representations are preserved or destroyed by common pooling operations. 

\textbf{Positioning and Relation to Prior Work.}
Our work connects three lines of research: permutation-invariant graph readouts \citep[e.g.,][]{zaheer2017deep,gilmer2017neural,corso2020pna}; WL-based expressivity theory for message-passing GNNs \citep[e.g.,][]{xu2019how,morris2019weisfeiler}; and representation-theoretic approaches to equivariance \citep[e.g.,][]{maron2019invariant,kondor2018generalization}.  While prior work has largely used representation theory as a \emph{design principle} to construct equivariant architectures or higher-order networks, here, we use it as an \emph{analytic tool} to characterize what standard pooling readouts preserve or destroy in order to understand precisely what existing models compute and why they fail.  This leads to an interpretable view of global pooling as a symmetry-induced projection and isolates a readout-level bottleneck that is complementary to WL-type encoder limitations.  A more detailed discussion of related work appears in Appendix~\ref{app:related_work}.

\textbf{Escaping the Collapse.}
We propose a family of projector-based invariant readouts that preserve permutation invariance while retaining symmetry-aware information.  The idea is to apply graph-dependent equivariant projectors that decompose node embeddings into invariant channels, followed by simple nonlinear invariant summaries of each channel.  This allows information that is provably invisible to global pooling to survive into the graph representation.

Empirically, replacing only the readout while keeping the encoder fixed enables GNNs to separate WL-hard graph pairs that standard pooling collapses, and consistently improves or matches performance on real-world benchmarks.

\textbf{Contributions.} Our specific contributions are:
\begin{itemize}[noitemsep]
\item We identify all linear permutation-invariant GNN readouts as instances of the Reynolds projection and show that they eliminate all non-trivial symmetry components.
\item We prove that any linear permutation-invariant readout necessarily factors through this projection, inducing an encoder-independent expressivity bottleneck.
\item We introduce projector-based invariant readouts that retain symmetry while recovering information provably invisible to global pooling.
\item We demonstrate improved graph discrimination and prediction performance without increasing message-passing complexity.
\end{itemize}

Together, these results highlight the readout as a crucial and under-studied component of GNN architectures and provide a principled path toward more expressive graph-level representations.

\section{Background: Representation Theory for Graph Symmetries} \label{sec:background}

Representation theory studies groups through their actions on vector spaces. In this paper, the groups of interest are graph automorphism groups, which act by permuting node indices.  We briefly recall the essential concepts needed for our construction and refer to \citet{serre_RT} for details.  All proofs and derivations in this work appear in Appendix \ref{app:proofs}.

We restrict our attention to finite groups and finite-dimensional complex vector spaces, since representations over $\mathbb{C}$ admit a  clean decomposition theory.  Although our analysis is carried out over $\mathbb{C}$, all constructions used in practice reduce to real-valued operations; see Appendix \ref{app:realvscomplex}.

Let $H$ be a finite group and $U$ be a complex (or real) vector space.  A \emph{(linear) representation} of $H$ on $U$ is a homomorphism $\rho: H \to \gl(U)$, where $\gl(U)$ denotes the group of invertible linear maps on $U$ (the group of isomorphisms of $U$); i.e., $\rho(h_1h_2)=\rho(h_1)\cdot \rho(h_2)$ for all $h_1,h_2 \in H.$  The dimension of $U$ is called the \emph{degree} of the representation.

Two representations $\rho_1: H \to \gl(U_1)$ and $\rho_2: H \to \gl(U_2)$ are \emph{isomorphic} if there exists a linear isomorphism $\tau: U_1 \to U_2$ such that $\tau \circ \rho_1(h) = \rho_2(h) \circ \tau$ for all $h \in H.$

\textbf{Decomposing Representations.}  A subspace $W\subseteq U$ is \emph{$H$-invariant} if $\rho(h)w\in W$ for all $h\in H$ and $w\in W$.  Restricting $\rho$ to such a subspace yields a \emph{subrepresentation}.  A representation is said to be \emph{irreducible} if it has no non-trivial invariant subspaces.

For finite groups, every representation decomposes into irreducible components:

\begin{theorem}[Maschke Decomposition, \citet{serre_RT}, Theorem 2] \label{thm:maschke}
Every representation of a finite group is completely reducible, that is, it can be written as a direct sum of irreducible subrepresentations.
\end{theorem}

While such decompositions are generally not unique, there exists a canonical and unique coarser decomposition that collects equivalent irreducible components, which is the \emph{isotypic (canonical) decomposition} and is central to our work.

A key tool is the \emph{character} of a representation, defined by $\chi_\rho(h):=\mathrm{tr}(\rho(h))$.  Characters determine representations up to isomorphism (\cref{thm:characters}) and allow explicit construction of invariant projectors.

Let $\{\rho_i:H\to\gl(W_i)\}_{i=1}^m$ be a complete set of pairwise non-isomorphic  irreducible representations of $H$ (of which there are finitely many \citep[Theorem 7]{serre_RT}), with characters $\chi_i$, ordered such that $\rho_1$ is the trivial representation (i.e., the one sending $h \mapsto 1 \in \gl(\IC) \cong \IC \setminus \{0\}$ for all $h\in H$). Then we have the following result.

\begin{theorem}[Canonical Decomposition, \citet{serre_RT}]
\label{thm:isotypic}

Any representation $U$ decomposes uniquely as $U = U_1 \oplus \ldots \oplus U_m$, where $U_i$ is the sum of irreducible subrepresentations isomorphic to $\rho_i$.  The corresponding projector is
    \begin{align} \label{eq:character_proj}
        p_i= \frac{\dim(W_i)}{|H|} \sum_{h \in H} \overline{\chi_i(h)} \ \rho(h),
    \end{align}
    where the overline indicates complex conjugation.
\end{theorem}
This decomposition is independent of all choices and isolates symmetry components that transform independently under $H$.

\section{The Readout Bottleneck in Graph Neural Networks}
\label{sec:bottleneck}
In this section, we show that common linear permutation-invariant readouts such as sum or mean pooling induce a fundamental bottleneck in expressivity.  While message passing layers can encode rich, symmetry-aware information in non-trivial representation components \cite{gilmer2017neural}, linear invariant readouts necessarily collapse these components, regardless of the expressivity of the encoder.

\subsection{Group Averaging Suppresses Non-Trivial Representations}

We first express the procedure of taking the mean in a representation theoretic framework. To this end, let us fix a finite group $H$, finite-dimensional complex vector space $U$ and representation $\rho:H \to \gl(U)$.
\begin{definition}[Reynolds Operator, \citet{derksen_kemper}, p.~38]
The \emph{averaging} or \emph{Reynolds operator} associated with $\rho$ is the map $p_{\text{avg}}:=p_1:U\to U_1$ (\cref{thm:isotypic}).
\end{definition}
For the trivial representation, we have $\dim(W_1)=1$ and $\chi_1(h)=1 \ \forall\ h \in H$. Hence, for $i=1$, \cref{eq:character_proj} simplifies to 
\begin{align}
    p_{\text{avg}}=\frac{1}{|H|} \sum_{h \in H} \rho(h). \label{eq:avgop}
\end{align}
In \cref{subsec:factorization} we will show that this operator is indeed part of sum and mean pooling. 

The importance of this statement lies in the fact that the averaging operator acts as a projection onto the trivial component $U_1$. In the context of GNNs, this means that any information encoded in non-trivial representations is irreversibly lost at the readout stage.  We shall see that this leads to restricted expressivity of GNNs.

\subsection{A Factorization Result for Linear Invariant Readouts} \label{subsec:factorization}

\emph{Permutation invariance} is an important and usually desirable property in machine learning. Intuitively, it specifies that the outcome of an invariant map should not depend on symmetric changes to the domain.  Specifically in our representation theory context, we are interested in the following setting. 
\begin{definition}
    Consider a representation $\rho: H \to \gl(U)$ and a linear map $f:U \to U'$ for some vector space $U'$. Then $f$ is called $H$\emph{-invariant} if $$f(u)=f(\rho(h) \, u) \quad \forall\ h \in H,\ u \in U.$$
\end{definition}

\begin{theorem} \label{thm:factorization}
    Linear $H$-invariant maps factor through the averaging operator, i.e., for any such $f:U \to U'$, we can find a linear $f':U_1 \to U'$ with $f=f' \circ p_{\text{avg}}.$
\end{theorem}

Therefore, any linear $H$-invariant readout depends only on the projection of the node embeddings onto the trivial subspace $U_1$ and discards all components orthogonal to this fixed subspace. In many common settings, this fixed subspace is one-dimensional and the projection coincides with taking the mean of the node features.  This result applies to any linear readout acting on equivariant node features, including those obtained from deep message passing architectures.

\subsection{Linear Invariant Readouts Induce an Expressivity Bottleneck}

We now study the implications of the above representation-theoretic results on GNNs. Message passing layers in GNNs are permutation-equivariant by design. As a result, intermediate node features transform according to representations of the permutation group.
The readout, however, is typically permutation-invariant, making it the only stage at which equivariant information is collapsed into a graph-level representation.

Consider a graph $G = (V,E)$ with $|V|$ nodes and assign a $d$-dimensional feature vector to each vertex.  After the final message passing layer, these yield a feature matrix $M \in \IR^{|V| \times d}$. Permutation-equivariant of message passing means that applying a permutation $\sigma$ to the nodes of $G$ will apply $\sigma$ also to the columns of the feature matrix $M \in \IR^{|V| \times d}$.

Let $H:=S_{|V|}$ (or a subgroup thereof) and $U:=\IR^{|V| \times d}$ (see Appendix~\ref{app:theory} for a discussion on real representations). Then $H$ acts on $V$ by permuting nodes, yielding a representation $\rho: H \to \gl(U)$ by mapping $h \in H$ to the map that permutes rows of matrices $M \in U$ according to $h$. Concretely, writing the rows of $M$ as $m_1, \dots, m_{|V|}$ we have
$$
\rho(h): M = \begin{pNiceMatrix}
    m_1 \\ \vdots \\ m_{|V|}
\end{pNiceMatrix} \mapsto \begin{pNiceMatrix}
    m_{h(1)} \\ \vdots \\ m_{h(|V|)}
\end{pNiceMatrix}.
$$

Let $f:U \to U'$ be a linear, $H$-invariant readout for some vector space $U'$ and let $U_1 \subset U$ be the trivial component of $U$ identified in \cref{thm:isotypic}. Then \cref{thm:factorization} implies:

\begin{corollary} \label{cor:bottleneck}
Let $f : U \to U'$ be a linear, $H$-invariant map, and let $U_1 \subset U$ denote the trivial isotypic component. Then $f$ factors through the projection $p_\text{avg}$ onto $U_1$.
\end{corollary}
In particular, after applying the readout, we cannot distinguish feature vectors that differed only in their $U_2 \oplus \dots \oplus U_m$ component.  This represents a significant bottleneck in the expressivity of GNNs: In practice, the group $H$ is usually the full symmetry group $S_{|V|}$, since global mean or sum pooling is invariant under all node permutations. Thus, by Corollary \ref{cor:bottleneck} the linear pooling suppresses all non-trivial isotypic components of the equivariant feature representation, i.e., every node feature that actually changes under node permutations. While message passing can encode structural information into non-trivial representation components, linear fully invariant readouts necessarily discard this information.  

\begin{figure}[h]\centering
\begin{subfigure}[b]{0.2\textwidth}
\centering
    \begin{tikzpicture}[x=0.8cm,y=0.3cm]
   \clip (-2,-2.6) rectangle (2,2.4);
    \node[inner sep=1.5pt,circle,draw,fill,label=below:$a$] at (0.4,-0.9) {};
    \node[inner sep=1.5pt,circle,draw,fill,label=above:$d$] at (-0.4,0.9) {};
    \node[inner sep=1.5pt,circle,draw,fill,label=below:$f$] at (-0.4,-0.9) {};
    \node[inner sep=1.5pt,circle,draw,fill,label=above:$c$] at (0.4,0.9) {};
    \node[inner sep=1.5pt,circle,draw,fill,label=right:$b$] at (1.4,0) {};
    \node[inner sep=1.5pt,circle,draw,fill,label=left:$e$] at (-1.4,0) {};
    \draw (1.4,0) -- (0.4,-0.9) -- (-0.4,-0.9) --  (-1.4,0) -- (-0.4,0.9) -- (0.4,0.9) -- (1.4,0);
\end{tikzpicture}
\caption{Graph $G^{(1)}$}
\end{subfigure} \hspace{.5cm} \begin{subfigure}[b]{0.2\textwidth}
\centering
    \begin{tikzpicture}
    [x=0.8cm,y=0.3cm]
    \clip (-2,-2.6) rectangle (2,2.4);
    \node[inner sep=1.5pt,circle,draw,fill,label=below:$a$] at (0.4,-0.9) {};
    \node[inner sep=1.5pt,circle,draw,fill,label=above:$d$] at (-0.4,0.9) {};
    \node[inner sep=1.5pt,circle,draw,fill,label=below:$f$] at (-0.4,-0.9) {};
    \node[inner sep=1.5pt,circle,draw,fill,label=above:$c$] at (0.4,0.9) {};
    \node[inner sep=1.5pt,circle,draw,fill,label=right:$b$] at (1.4,0) {};
    \node[inner sep=1.5pt,circle,draw,fill,label=left:$e$] at (-1.4,0) {};
    \draw (1.4,0) -- (0.4,-0.9) -- (0.4,0.9) -- (1.4,0);
    \draw (-1.4,0) -- (-0.4,-0.9) -- (-0.4,0.9) -- (-1.4,0);
\end{tikzpicture}
\caption{Graph $G^{(2)}$}
\end{subfigure}
\caption{The graphs in Example \ref{ex:1}}
\label{fig:ex1}
\end{figure}
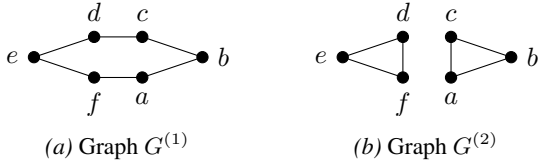

\begin{example} \label{ex:1}
    Consider the two graphs in \cref{fig:ex1}. These graphs are indistinguishable by 1-WL and hence by standard message-passing GNNs with sum pooling, indicating that a fully symmetric readout layer is too coarse of a choice. On the other hand, the readout should at least respect the symmetries of the graphs. We thus consider $H^{(1)}=D_6 \subset S_6$ which reflects the symmetry of $G^{(1)}$, and $H^{(2)} = C_2 \times D_3 \times D_3 \subset S_6$ reflecting the symmetry group of $G^{(2)}$ ($C_2$ for swapping the triangles, and $D_3$s for applying symmetries to them). In Lemma \ref{lem:irrep_decompositions}, we show that the following are corresponding canonical decompositions of $U=\IR^6$: Let $U^{(i)}$ denote the copy of $U$ on which $H^{(i)}$ acts with the basis given by the labelling in \cref{fig:ex1} and let $\langle \bullet \rangle$ denote the linear span of a set of vectors. Then
    \begin{align*}
        U^{(1)}=&\textcolor{blue}{\langle(1,1,1,1,1,1)^\top \rangle} \oplus \langle(1,-1,1,-1,1,-1)^\top \rangle\\
        &\oplus \langle (1,1,0,-1,-1,0)^\top,(0,1,1,0,-1,-1)^\top \rangle\\
        &\oplus \langle (1,-1,0,1,-1,0)^\top,(0,1,-1,0,1,-1)^\top \rangle \\
        U^{(2)}=&\textcolor{blue}{\langle(1,1,1,1,1,1)^\top \rangle} \oplus \langle(1,1,1,-1,-1,-1)^\top \rangle\\
        &\oplus \langle u_1,u_2,u_3,u_4 \rangle \qquad \qquad \text{with}\\
        u_1=&(1,-1,0,0,0,0)^\top \quad u_3=(0,0,0,1,-1,0)^\top \\
        u_2=&(0,1,-1,0,0,0)^\top \quad u_4=(0,0,0,0,1,-1)^\top 
    \end{align*}
    Since $G^{(1)}$ and $G^{(2)}$ are WL-equivalent, message passing produces identical update trajectories on both graphs \citep[p.~14]{xu2019how}. Even if an equivariant encoder were to populate the non-trivial representation components differently for the two graphs, any linear $H^{(1)}$- (resp.~$H^{(2)}$-) invariant readout would project these differences away and collapse both representations onto the trivial component. In particular, the trivial (blue) component is identical for $G^{(1)}$ and $G^{(2)}$, explaining why standard GNNs cannot distinguish them. In contrast, a readout that accesses non-trivial irreducible components would retain graph-dependent information and render $G^{(1)}$ and $G^{(2)}$ distinguishable.
\end{example}
This observation is the basis for our proposed architectures.

\section{Projector-Based Invariant Readouts}
\label{sec:methods}
We have shown that commonly used linear permutation-invariant readouts, such as sum or mean pooling, effectively reduce node representations to their projection onto the fixed subspace, which in many settings coincides with a global average over nodes.  As a consequence, substantial information produced by message passing layers is discarded at the readout stage.  Moreover, as illustrated in Example~\ref{ex:1}, a representation-theoretic analysis reveals structures that are invisible both to standard readouts and, in certain cases, to message passing itself, enabling the distinction of WL-equivalent graphs.

We now introduce graph-dependent, symmetry-aware readout constructions that exploit this representation-theoretic structure.  Our proposed readouts are invariant under graph automorphisms, yet avoid the averaging-induced bottleneck of classical pooling. 

\subsection{Overview and Principles of Design}  
The readout bottleneck arises from the combination of linearity and global $H$-invariance in standard pooling functions. While breaking linearity alone is not sufficient to overcome this bottleneck, escaping the collapse requires moving beyond linear invariant aggregation.  Since invariance---at least with respect to the automorphism group $\mathrm{Aut}(G)$, which captures the symmetries of the graph---is essential for well-defined graph representations, we instead construct \emph{graph-dependent}, nonlinear $\mathrm{Aut}(G)$-invariant readouts.

Concretely, our readout is a function of both the graph $G$ and the final-layer node embeddings $M$.  It first decomposes $M$ into $\mathrm{Aut}(G)$-invariant channels via representation-theoretic projectors, and then summarizes each channel using simple nonlinear invariant statistics.  This design preserves permutation invariance while retaining information from non-trivial symmetry components that we have shown to be discarded by linear pooling.

\subsection{Equivariant Projectors and Channel Decomposition} 

We now describe how to construct a symmetry-aware decomposition of node representations that will serve as the foundation for our readout.  The goal is to decompose the node feature space into multiple permutation-invariant channels that retain information from non-trivial symmetry components, while remaining computationally feasible.

We consider a graph $G=(V,E)$ with $n=|V|$ nodes and let $H=\Aut(G)$ act on $\IR^{n}$ by permuting the standard basis. Our goal is to define an $H$-invariant decomposition of $U$ together with projection maps onto the respective spaces. The canonical projectors $p_i$ defined in \cref{thm:isotypic} resp.~\cref{app:realvscomplex} isolate irreducible symmetry components and provide a principled way to access symmetry-aware information.  However, this decomposition is expensive to compute: it requires enumerating all irreducible representations of $H$ as well as their characters, which is computationally prohibitive for general graphs.  

Instead, we adopt a computationally efficient alternative that recovers a refinement of the canonical decomposition using only linear algebra.  The key idea is to construct a generic linear operator that commutes with the action of $H$. The eigenspaces of such an operator are necessarily invariant under $H$, and therefore provide symmetry-aware channels.

\textbf{Constructing an Equivariant Operator.}  We construct a linear map $S$ that commutes with the action of $H$, i.e., we want $S: U \rightarrow U$ such that $S \circ \rho(h)=\rho(h) \circ S$ for all $h \in H$, meaning that $S$ is equivariant with respect to the group action. We obtain $S$ using ideas from \citep{maron2019invariant,keriven2019universal}.

Let $[n]:=\{1,\dots,n\}$ and consider the induced action of $H$ on ordered pairs $[n]\times[n]$: $h\cdot(i,j) := (h(i),\, h(j)).$  This action partitions $[n]\times [n]$ into orbits $\mathcal{O}_1,\dots,\mathcal{O}_m$,
where two pairs belong to the same orbit if and only if they are related by a graph automorphism. Each orbit captures a distinct symmetry pattern between node pairs.
For each orbit $\mathcal{O}_t$, we define its indicator matrix $M_t \in \IR^{n\times n}$ by
$$ (M_t)_{ij} :=
\begin{cases}
1 & (i,j)\in \mathcal{O}_t,\\
0 & \text{otherwise}.
\end{cases}
$$
We fix a random seed, sample coefficients $c_1,\dots,c_m\in\IR$ deterministically, and define
\begin{equation}
\label{eq:S-orbitals}
S := \sum_{t=1}^{m} c_t \cdot \frac{M_t+M_t^\top}{2}.
\end{equation}
By construction, $S$ is symmetric and depends only on the automorphism structure of the graph.

\begin{lemma} \label{lem:S_com_action}
    $S$ commutes with the action of $H$.
\end{lemma}

Intuitively, $S$ is a generic symmetry-respecting operator: it treats node pairs equally if and only if they are equivalent under graph automorphisms.  As a result, its eigenspaces yield symmetry-invariant directions in the node feature space.

\textbf{Eigendecomposition and Channel Projectors.}  Since $S$ is symmetric, it admits an eigendecomposition $S=Q\Lambda Q^\top$ where $Q$ is orthogonal and $\Lambda=\mathrm{diag}(\lambda_1,\dots,\lambda_n)$.  We group indices with equal eigenvalues (up to a numerical tolerance) into disjoint blocks
$I_1,\dots,I_q\subseteq [n]$.  Each block corresponds to an eigenspace of $S$, which we interpret as a symmetry-aware channel. For each block $I_\alpha$, we define the orthogonal projector 
$\inlineequation[eq:eigenspace-projector]{
P_\alpha := Q \Lambda_\alpha Q^\top
}$, where $\Lambda_\alpha$ is the diagonal matrix that has entries 1 for all indices belonging to $I_\alpha$ and 0 otherwise.

\begin{lemma} \label{lem:eigenspaces_invariant}
    The maps $P_\alpha$ are orthogonal projectors onto the eigenspace of $S$ and their images $\{P_\alpha(U): 1 \le \alpha \le q \}$ form an orthogonal decomposition of $U$. Moreover, every one of these spaces is invariant under $H$.
\end{lemma}
Thus, the eigenspaces provide an $H$-invariant channel decomposition of the node representation space.  By \cref{facts:candecomp}, this decomposition refines the isotypic decomposition: each channel lies entirely within a single isotypic component, but may further split it.

\begin{example} \label{ex:2}
    Revisiting the example of \cref{fig:ex1}, possible generic matrices $S$ for $G^{(1)}$ and $G^{(2)}$ are
    \begingroup
    \setlength{\arraycolsep}{3pt}
    $$S^{(1)}=\begin{pNiceMatrix}
        1&2&3&5&3&2\\
        2&1&2&3&5&3\\
        3&2&1&2&3&5\\
        5&3&2&1&2&3\\
        3&5&3&2&1&2\\
        2&3&5&3&2&1
    \end{pNiceMatrix}, \quad S^{(2)}=\begin{pNiceMatrix}
        5&2&2&1&1&1\\
        2&5&2&1&1&1\\
        2&2&5&1&1&1\\
        1&1&1&5&2&2\\
        1&1&1&2&5&2\\
        1&1&1&2&2&5
    \end{pNiceMatrix}.$$
    \endgroup
    Their eigenvalues are $\{16,-2,-5,-5,1,1\}$ and $\{12,6,3,3,3,3\}$, respectively.  The eigendecompositions of $U^{(1)}$ and $U^{(2)}$ are precisely the canonical decompositions presented in Example \ref{ex:1}. Thus, our proposed construction recovers the full representation-theoretic structure.
\end{example}

\textbf{Practical Considerations.}  We sort the projectors $P_\alpha$ by the criteria $\bigl(\tr(P),\ \tr(P L(G)),\ \tr(P A(G))\bigr)$ in descending lexicographic order, where $L(G)$ and $A(G)$ denote the graph Laplacian and adjacency matrix of $G$. We truncate to retain the first $B$ channels and pad with zeros if necessary. The projectors are computed once for each graph.

\subsection{Invariant Summarization}

Having decomposed the node representation space into symmetry-aware channels, we now describe how to summarize each channel into permutation-invariant features suitable for graph-level prediction.

Let $M\in\IR^{n\times d}$ denote the node embeddings produced by the final message passing layer for a graph $G$, and let $P_1,\dots,P_B$ be the sorted and truncated block projectors constructed as above.  For each block projector $P_\alpha$ we define the projected embedding,
\(
M_\alpha := P_\alpha M \in \IR^{n\times d}.
\)
Each $M_\alpha$ isolates the contribution of $M$ lying in a specific symmetry-aware subspace of the node representation space.

To summarize the information contained in each block while preserving permutation invariance, we extract the following nonlinear scalar features from each block $\alpha$:
\begin{align*}
s_1 (\alpha) &:= \left\|\sum_{j=1}^n (M_\alpha)_{j,:}\right\|_2 \! \! \! = \! \left(\sum_{j=1}^n \! \left( \sum_{k=1}^n \bigl(M_\alpha\bigr)_{j,k} \! \right)^2 \right)^{\frac 1 2}\!\!,\\
s_2(\alpha) &:= \|M_\alpha\|_{\mathrm{F}} = \! \left(\sum_{j=1}^n \sum_{k=1}^n \bigl(M_\alpha\bigr)_{j,k}^2 \right)^{\frac 1 2}\!\!,\\
s_3(\alpha) &:= \frac{1}{n}\! \sum_{j=1}^n \|(M_\alpha)_{j,:}\|_2= \! \frac{1}{n} \! \sum_{j=1}^n \! \left( \sum_{k=1}^n \bigl(M_\alpha\bigr)_{j,k}^2 \! \right)^{\frac 1 2}\!\!.
\end{align*}

Intuitively, $s_1$ measures the magnitude of the block-summed signal, $s_2$ measures the total block energy, and $s_3$ captures the typical row-scale within the block. These features capture complementary aspects of the projected signal.

In addition, we include a low-dimensional summary of the mean block embedding. Concretely, we define
\[
\mu(\alpha) := \frac{1}{n}\sum_{j=1}^n (M_\alpha)_{j,:}\in\IR^{d},
\]
and apply a fixed random projection matrix $R\in\IR^{d\times r}$ (parameter \texttt{rp\_dim=r} stored as a non-trainable buffer), to obtain $\mu(\alpha)^\top R \in \mathbb{R}^r$.  This projection provides an efficient, $H$-invariant sketch of the mean feature vector without introducing additional trainable parameters.

The resulting $(3+r)$-dimensional block feature vectors are 
\begin{equation*}
\label{eq:block-vector}
\psi_\alpha(G,M)
:=
\Bigl(
s_1(\alpha),\
s_2(\alpha),\
s_3(\alpha),\
\mu(\alpha)^\top R
\Bigr).
\end{equation*}
Finally, the full readout is the concatenation of the block features across all retained channels:
\begin{equation*}
R_{\mathrm{iso}}(G,M):=\mathrm{Concat}\bigl(\psi_1(G,M),\dots,\psi_B(G,M)\bigr).
\end{equation*}
This invariant summarization step is crucial: by applying nonlinear functions independently to each symmetry-aware channel, the readout preserves information from non-trivial symmetry components that we have shown to be inaccessible to linear pooling, while remaining fully permutation invariant and compatible with standard GNN pipelines.

An optional centering step removes the globally invariant component that is visible to standard pooling, without affecting equivariance; details are deferred to Appendix~\ref{app:centering}.

\subsection{Final Readout Architecture}

Figure~\ref{fig:schema} illustrates the proposed readout architecture applied to a graph $G$ with final-layer node representations $M \in \mathbb{R}^{|V|\times d}$.  The readout is designed as a two-stage procedure that preserves permutation invariance while avoiding the averaging-induced bottleneck of standard pooling.

In the first stage, the node representation matrix $M$ is decomposed into a collection of symmetry-aware channels $\{M_\alpha\}_{\alpha=1}^B$ via a family of graph-dependent linear projections $P_\alpha$.  These projections are constructed from the automorphism structure of $G$, and isolate distinct invariant subspaces of the node representation space.  Importantly, this step does not aggregate across channels, but instead exposes multiple components that would otherwise be collapsed by global pooling.

In the second stage, each projected component $M_\alpha$ is summarized independently using nonlinear, permutation-invariant statistics, and the resulting summaries are concatenated to form the final graph-level representation $R_{\mathrm{iso}}(G,M)$.  This nonlinear aggregation allows the readout to retain information from non-trivial symmetry components while remaining invariant under relabelings of the nodes.

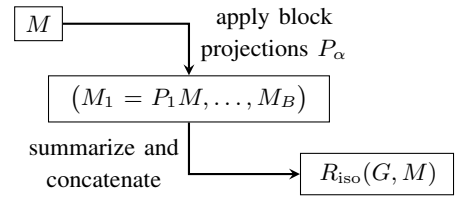
\begin{figure}[h]
    \centering
    \vspace{-.2cm}
    \begin{tikzpicture}[x=1cm,y=1cm]
	\clip(-2.9,-2.3) rectangle (3.6,0.3);
        \node[draw, text width=0.4cm,align=center] (a) at (-2,0) {\small $M$};
        \node[draw, text width=3.5cm,align=center] (b) at (0,-1) {\small $\bigl(M_1 = P_1 M,\dots, M_B\bigr)$};
        \node[draw, text width=1.8cm,align=center] (c) at (2.5,-2) {\small $R_{\mathrm{iso}}(G,M)$};
        \draw [-stealth,thick] (a.east) -- (0,0) -- (b.north) node[midway, text width=2cm,align=center, right,yshift=2mm] {\small apply block projections $P_\alpha$};
        \draw [-stealth,thick] (b.south) -- (0,-2) node[midway, text width=2cm,align=center, left,yshift=-2mm] {\small summarize and concatenate} -- (c.west) ;
    \end{tikzpicture}
    \caption{Schematic representation of our readout} 
    \label{fig:schema}
    \vspace{-.2cm}
\end{figure}

In Lemma \ref{lem:welldef}, we prove that the resulting readout is independent of the chosen node labeling. As a result, the proposed architecture can be viewed as a drop-in replacement for standard sum or mean pooling: it operates solely at the readout level, preserves permutation invariance, and enhances expressivity without modifying the message passing layers.

\section{Experiments}
\label{sec:experiments}

Our experimental design is structured to isolate the contribution of the proposed isotypic decomposition specifically at the readout stage, independently of encoder choice, optimization effects, or architectural complexity.  

We proceed in three tiers: (i) In a \textbf{training-free} separation suite, we freeze randomly initialized encoders and remove learning entirely.  In this setting, two non-isomorphic graphs can be separated only if the encoder-readout map is intrinsically sensitive to structural differences.  This directly targets the theoretically-identified failure mode: on highly symmetric and WL-hard pairs with constant features, message passing often produces node embeddings that are constant on large automorphism orbits, after which standard multiset readouts collapse all remaining information. (ii) We reintroduce \textbf{supervision} on symmetry-heavy benchmarks (SRG-16 and BREC RPC-lite), while preserving strict invariance requirements by testing generalization to unseen node relabelings.  Here, success requires that the readout exposes permutation-invariant signals that distinguish the graphs; if such information has already been erased by pooling, supervision cannot recover it. (iii) Finally, we evaluate \textbf{downstream tasks} to assess whether the proposed readout integrates into standard pipelines. 

Across all tiers, the consistent pattern---chance-level performance for multiset-based pooling and strong separation from the isotypic readout---supports the central claim: in symmetry-heavy regimes, the readout is often the decisive expressivity bottleneck rather than the message-passing encoder.

Code is available upon request. Implementation details, computing infrastructure, and runtimes are given in Appendix~\ref{app:runtime}.

\subsection{Training-Free Separation on WL-Equivalent Pairs}
\label{subsec:trainfree_sep}

\textbf{Setup.}
We evaluate a training-free separation suite to isolate \emph{readout-level} information loss.
For each WL-equivalent but non-isomorphic pair $(G,G')$, we sample a random encoder and a fixed readout, apply independent random node relabelings,
and compare $\ell_2$-normalized graph embeddings \emph{after the readout} via cosine similarity $\cos(z(G),z(G'))$.
No parameters are trained. A pair is counted as \emph{separated} if the mean similarity across seeds is $<0.95$.
We consider 36 instances from 4 WL-hard families: $2C_k$ vs.\ $C_{2k}$; CFI on $K_3/K_4$ \cite{cai1992optimal}; helical ribbons; and Godsil--McKay switching on the Petersen graph \cite{godsil1982constructing}.

\textbf{Results and Interpretation.}
Standard multiset pooling collapses almost perfectly, with similarity $\approx 1.00$ across all tested encoders and pairs
(Table~\ref{tab:trainfree_main}).
In contrast, the isotypic readout separates 33/36 instances consistently across GIN \cite{xu2019how}, GraphSAGE \cite{hamilton2017inductive},
and PNA \cite{corso2020pna}.
Since no optimization is involved, this directly reflects the geometric failure mode predicted by the theory:
on these symmetry-heavy WL-hard pairs (with constant features), message passing produces node embeddings whose multiset statistics coincide,
so multiset readouts collapse, while the proposed graph-dependent decomposition exposes invariant structure that pooling cannot.
A GIN family breakdown is in Table~\ref{tab:tf_family_gin}; the remaining failures are the helical ribbons, suggesting a limitation of the current
blockwise statistics rather than invariance itself.

\textbf{Ablations.}
Training-free ablations (Appendix~\ref{app:trainfree_ablations}) isolate the roles of graph-dependent projectors and nonlinear blockwise summaries:
linearized variants collapse to pooling, while increasing the block budget yields a monotone gain that saturates at a small number of blocks.

\subsection{Supervised Symmetry-Heavy Benchmarks}
\label{subsec:srg16}

We now reintroduce supervision on symmetry-heavy tasks where invariance to node relabeling is mandatory and evaluation is
performed on held-out relabelings. In this regime, success depends on whether the readout exposes a permutation-invariant
yet graph-sensitive signal; if pooling collapses the representation, labels cannot recover it.

\textbf{SRG-16.}  SRG-16 is a binary classification task distinguishing the Shrikhande graph from the $4\times4$ rook graph \cite{seidel1979strongly}, with each example presented under a random node relabeling.  Success therefore requires permutation-invariant representations that still encode graph-level structure.

We evaluate GIN, GraphSAGE, GATv2~\cite{brody2021attentive}, PNA, and a graph transformer, swapping only the readout while
keeping the encoder and training protocol fixed. Across all encoders, multiset pooling remains near chance ($\approx 0.45$),
whereas the isotypic readout achieves perfect accuracy ($1.00$). This supervised analogue of the training-free suite shows
that the encoder can produce informative features, but symmetry-induced collapse at the readout prevents the classifier from
accessing them; changing only the readout resolves the problem, indicating that the readout is the limiting factor here.

\textbf{BREC (RPC-lite).}  BREC~\cite{brec} is designed as a broad distinguishability stress test via non-isomorphic pairs of increasing difficulty,
many of which are challenging for standard message-passing GNNs with multiset pooling. We therefore use BREC as a primary
paired-comparison evaluation of readout-level expressivity. We adopt a specific experimental set-up that we call RPC-lite: each pair $(G_0,G_1)$ is expanded into a
balanced binary classification task by sampling random node relabelings, training on a subset of permutations, and testing
on held-out permutations. This directly operationalizes the requirement ``permutation-invariant yet pair-sensitive'' under a
fixed compute budget.

All runs use constant node features and hold the encoder fixed while swapping only the readout.  Table~\ref{tab:brec_rpclite_main} reports average test accuracy. Across the whole BREC dataset, all multiset-based readouts remain at chance ($0.5$), while the isotypic readout achieves strong performance ($0.901$ ). As a supervised analogue of the training-free separation tests, BREC RPC-lite shows that access to labels and end-to-end optimization does not overcome the fixed-subspace bottleneck of pooling: when the readout collapses structurally distinct embeddings, supervision has no signal to exploit. In contrast, the isotypic readout exposes graph-dependent invariant structure to the classifier, enabling generalization across relabelings and confirming that the readout—not the loss or encoder—is the limiting factor in this regime.

Ablations are reported in Table~\ref{tab:brec_blocksweep}.

Due to the computational cost of symmetry-aware projectors, we include a partial run of the official BREC protocol in Appendix~\ref{subsec:brec_official_partial} as a supplementary protocol-level check, which further confirms that replacing only the readout improves distinguishability under the benchmark’s standard evaluation.

\vspace{-.1cm}
\subsection{Downstream Analysis: Benchmarks, Regression, and Classification (GIN only)}
\label{subsec:downstream}

The previous experiments deliberately target symmetry-heavy discrimination, where standard pooling is most brittle.  We finally test whether the proposed readout integrates into standard supervised pipelines on real-world tasks, and whether it enhances performance even when symmetry is not the dominant challenge.  We fix the encoder to GIN and compare three readouts: sum pooling, the pure isotypic readout, and a combination readout that combines sum and isotypic channels.

\textbf{Proteins and Enzymes.}  We evaluate on the TUDatasets PROTEINS and ENZYMES \cite{morris2020tudataset} graphs with a spectral regression target derived from global structure.  

\textbf{Proteins and Enzymes: Setup, Results, and Interpretation.}
For each graph $G$, we compute the normalized Laplacian $L = I - D^{-1/2} A D^{-1/2}$, discard the trivial eigenvalue at $0$,
and regress the next $k=8$ eigenvalues.  This target is permutation-invariant and global, providing a controlled setting in which the readout must summarize
structure beyond local neighborhoods.  Models are trained end-to-end with MSE loss, and we report test MAE and $R^2$ using the checkpoint with best validation MAE.

\vspace{-.1cm}
\begin{table}[h]
\centering
\caption{Spectral regression with GIN. Lower MAE is better; higher $R^2$ is better.}
\label{tab:spectral_main}
\small
\begin{tabular}{lcccc}
\toprule
& \multicolumn{2}{c}{\textbf{PROTEINS}} & \multicolumn{2}{c}{\textbf{ENZYMES}} \\
\cmidrule(lr){2-3}\cmidrule(lr){4-5}
\textbf{Readout} & \textbf{MAE} & $\mathbf{R^2}$ & \textbf{MAE} & $\mathbf{R^2}$ \\
\midrule
sum               & 0.0869 & 0.8877 & 0.0978 & 0.4360 \\
isotypic          & \textbf{0.0824} & \textbf{0.8940} & \textbf{0.0867} & \textbf{0.6181} \\
sum + isotypic    & 0.0858 & 0.8896 & 0.0883 & 0.5853 \\
\bottomrule
\end{tabular}
\end{table}
\vspace{-.1cm}

Table~\ref{tab:spectral_main} reports performance on both datasets. It shows that the isotypic readout improves over sum pooling on PROTEINS and yields a substantial gain on ENZYMES in both MAE and $R^2$, indicating increased sensitivity to global spectral structure.  The combined readout performs comparably, suggesting that the isotypic channel captures information complementary to pooling.

\textbf{ZINC and MolHIV.}  We further evaluate on two additional standard molecular benchmarks.  ZINC \cite{dwivedi2023benchmarking} is a molecular regression task with strong atom- and bond-level features, while MolHIV (OGB) \cite{hu2020open} is a binary molecular classification task.

\textbf{ZINC and MolHIV: Setup, Results, and Interpretation.} For ZINC, we follow the standard PyG split
and report test MAE; for MolHIV, we use the official split and report test ROC--AUC.

\begin{table}[h]
\centering
\caption{Molecular benchmarks with GIN. ZINC: MAE (lower is better). MolHIV: ROC--AUC (higher is better).}
\label{tab:molecules_main}
\small
\begin{tabular}{lcc}
\toprule
\textbf{Readout} & \textbf{ZINC MAE} & \textbf{MolHIV ROC--AUC} \\
\midrule
sum               & \textbf{0.3816} & 0.7118 \\
isotypic          & 0.6117 & 0.7761 \\
sum + isotypic    & 0.3987 & \textbf{0.7916} \\
\bottomrule
\end{tabular}
\end{table}

Table~\ref{tab:molecules_main} summarizes the results: On ZINC, a pure isotypic readout is not competitive with sum pooling, consistent with the effectiveness of established
local aggregation heuristics.  On MolHIV, however, augmenting sum pooling with the isotypic channel yields a clear improvement in ROC--AUC.

Across downstream tasks, we see that the isotypic readout is not a drop-in replacement for pooling,
but acts as a complementary invariant descriptor.  When global structure is predictive (as in spectral regression or MolHIV), the additional symmetry-aware information improves performance, while in regimes dominated by local chemistry (ZINC), standard pooling remains preferable.

\emph{Remark.}  While our theory covers linear readouts, we observe the same collapse for standard nonlinear invariant readouts, suggesting a more general limitation of permutation-invariant aggregation.

\section{Discussion}

We show that GNN readouts impose an expressivity bottleneck independent of message passing: any linear permutation-invariant readout collapses node embeddings and discards symmetry-aware information.  Our theory makes this collapse explicit, and our experiments show that it prevents standard pooling from exploiting structure already present in the embeddings.  A graph-dependent invariant readout that decomposes embeddings into symmetry-aware channels and applies nonlinear summaries exploits this otherwise lost information, yielding strong gains on symmetric and WL-hard graphs.  Together, these results identify the readout as a decisive component of GNN expressivity. 

\textbf{Limitations and Future Work.} Computing automorphisms and our channel decompositions scale poorly for large or highly symmetric graphs, motivating approximate or sparse constructions.  Our method also cannot separate graphs with identical symmetry groups, and must be complemented by the encoder or richer invariant summaries.  Finally, a deeper understanding of when message passing collapses and which invariant features suffice for recovery remains open. 

\section*{Acknowledgments}

A.M.~is supported by the EPSRC AI Hub on Mathematical Foundations of Intelligence:
An ``Erlangen Programme'' for AI [EP/Y028872/1].  A.W.~is funded by a London School of Geometry and Number Theory–Imperial College London
PhD studentship, which is supported by the Engineering and Physical Sciences Research Council
[EP/S021590/1].

\section*{Impact Statement}

This paper presents work whose goal is to advance the field of 
Machine Learning. There are many potential societal consequences 
of our work, none which we feel must be specifically highlighted here.

\bibliography{references.bib}
\bibliographystyle{icml2026}

\newpage
\appendix
\renewcommand{\thetable}{\thesection.\arabic{table}}
\onecolumn

\section{Related Work}
\label{app:related_work}
In this section of the Appendix, we provide a more detailed review of related research. Our work lies at the intersection of three research directions: (i) permutation-invariant learning on sets and graphs; (ii) expressivity theory for message-passing GNNs; and (iii) representation-theoretic perspectives on symmetry and invariance. We have grouped the related work accordingly.

\subsection{Permutation-Invariant Set Functions and Graph Readouts}

Most GNNs obtain graph-level representations by aggregating node embeddings with a permutation-invariant operator. This practice is often justified by results on invariant set functions: Deep Sets show that, under mild conditions, any permutation-invariant map can be written as $\rho(\sum_v \phi(x_v))$ for suitable functions $\phi,\rho$ \citep{zaheer2017deep}. This is typically used as justification for the widespread use of sum and mean pooling in graph classification.

Several more expressive readouts have been proposed, including sequence-to-sequence pooling \citep{vinyals2016order},
Janossy pooling \citep{murphy2019janossy}, Set2Set \citep{gilmer2017neural}, attention-based pooling, and richer aggregation schemes such as PNA \citep{corso2020pna}.  Our work fits into this line, but targets a specific and previously uncharacterized failure mode: multiset pooling can erase graph-dependent information that is still present in the node representation.

\subsection{Expressivity of Message Passing and WL Limitations}
A large literature relates message-passing GNNs to the Weisfeiler--Leman (WL) refinement procedure \citep{weisfeiler1968reduction}.  In particular, GIN \citep{xu2019how} establishes that standard message passing is upper bounded by 1-WL under common initializations, explaining the difficulty in working with many non-isomorphic graph families, including Cai--Fürer--Immerman constructions \citep{cai1992optimal} and strongly regular graphs.  These limits are usually framed as \emph{encoder} limitations.  Our contribution is complementary: we identify a regime where the encoder may in principle produce structured, symmetry-aware embeddings, but the readout forces a collapse through averaging, creating
an independent bottleneck at the aggregation stage.

One way to overcome WL limitations is to increase representational order.  Higher-order GNNs and WL-inspired models operate on node tuples or higher-order tensors, yielding stronger invariants \citep{morris2019weisfeiler,maron2019provably}. Invariant and equivariant architectures construct features that span rich spaces of graph invariants \citep{maron2019invariant}. Another route is to enrich node states with positional or structural encodings, obtained for example from graph distances \citep{li2020distance}, topological data analysis \citep{horn_topological_2021}, or spectral information \citep{lim2022signnet}. A complementary class of approaches considers the graphs together with marked subgraphs \citep{bevilacqua2022esan, bouritsas2020gsn, zhang2021nestedgnn,southern2025balancing}. These approaches can greatly improve expressivity, but often at substantial computational cost. Our goal is different: we keep standard message-passing encoders fixed and instead modify only the readout to expose graph-dependent invariant structure that global pooling discards.

\subsection{Representation Theory and Equivariance in Deep Learning}  
Representation-theoretic ideas have played a central role in modern equivariant learning,
from group-equivariant CNNs \citep{cohen2016group} to general frameworks for equivariant neural networks
\citep{kondor2018generalization} and E(3)/E(n)-equivariant geometric models \citep{satorras2021en}.
In graph learning, these ideas and representation theory in general typically appear as architectural constraints: networks are designed so that all layers transform equivariantly under a symmetry group, such as rotational symmetry \citep{cohen2018spherical}.

Our use of representation theory is different. Rather than building new equivariant layers, we apply representation theory as an \emph{interpretability and analysis tool}: to characterize how standard pooling acts on the symmetry structure of learned node embeddings and to explain why many GNNs fail on symmetric graph families.
Orbits and their indicator matrices provide a principled way to decompose node representations into invariant subspaces, yielding a readout-level intervention that remains compatible with standard message-passing backbones while directly addressing the averaging bottleneck revealed by the theory.

\section{Additional Theory and Proofs}
\label{app:theory}
This part of the Appendix contains additional theory on real, as compared to complex representation theory, as well as proofs for all statements that are not directly cited from the literature.

\subsection{Real vs.~Complex Representation Theory}
\label{app:realvscomplex}
We explain how the representation–theoretic results used in this paper carry over from complex to real vector spaces. The definitions of representations and irreducibility are unchanged; the only difference is that we now work with real vector spaces instead of complex ones.

The Maschke decomposition of representation into their irreducible representations continues to hold over $\IR$ and more generally over any field $K$ with $\mathrm{char}(K) \nmid |H|$ \citep[Theorem 10.8]{curtis_reiner}.  In fact, the proof given in \cite{serre_RT} over the complex numbers carries over directly to the reals.

Isotypic decompositions also exists over the reals \citep[p.~156]{procesi}. However, unlike in the complex case, characters need not be real-valued, and consequently one does not in general obtain closed-form expressions for the isotypic projectors analogous to \eqref{eq:character_proj}. The Reynolds operator \eqref{eq:avgop}, by contrast, is always defined over $\mathbb{R}$ and remains unchanged \citep[Example 2.2.3]{derksen_kemper}.

A sufficient condition for the existence of real-valued characters, and hence real isotypic projectors, is given by the following classical result.

\begin{theorem}[Frobenius--Schur, \cite{serre_RT}, Theorem 31] \label{thm:frobenius-schur}
Let $\rho:H \to \gl(U)$ be a representation. In order that $\chi_\rho$ is real-valued it is necessary and sufficient that $U$ has a symmetric bilinear form invariant under $H$.
\end{theorem}
A proof can be found in \citep[p.~107]{serre_RT}.

In particular, for permutation representations including the ones considered in this paper, this condition is always satisfied:

\begin{corollary} \label{cor:real}
    If a group $H$ acts on $U$ by permuting a basis, the character of the corresponding representation is real-valued.
\end{corollary}
\begin{proof}
By \cref{thm:frobenius-schur}, it suffices to find a symmetric bilinear form of $U\cong \IC^n$ that is invariant under $H$.  Such a form is given by the standard inner product (with respect to the basis permuted by $H$)
$$\langle u,w\rangle = \sum_{k=1}^n \overline u_k w_k \quad \forall u,w \in U,$$
since permuting the basis amounts to permuting the summands which leaves the sum unchanged.
\end{proof}

As a consequence, all canonical projectors and decompositions used in this paper admit realizations on the real subspace $U\cong\mathbb{R}^n$ of $\IC^n$.  We now explain how the representation–theoretic results stated over $\mathbb{C}$ are applied in the real-valued settings of Sections~\ref{sec:bottleneck} and~\ref{sec:methods}.

Let $H$ act on $\mathbb{C}^n$ by permuting the standard basis according to graph automorphisms and let $U \cong \mathbb{R}^n$ denote the corresponding real subspace of $\IC^n$. The canonical projectors $p_i$ defined in \cref{thm:isotypic} are real-valued by Corollary~\ref{cor:real} and therefore project $U$ to the real components of the complex canonical decomposition, yielding a real canonical decomposition $U=U_1 \oplus \dots \oplus U_m$. Consequently, all results stated in \cref{sec:background} for complex numbers also hold over the reals for the concrete representations considered in this paper.

\subsection{Proofs and Supplementary Facts}
\label{app:proofs}
We start by providing two additional statements, that give further intuition for the important representation theoretic notions. The first one demonstrates the eponymous role of characters as summaries and classifiers of their respective representations.

\begin{theorem}[Properties of Characters, \citet{serre_RT}, Corollary 2 and Theorem 5] \label{thm:characters}
    \begin{enumerate}
        \item[a)] Two representations with the same character are isomorphic.
        \item[b)] Let $\chi$ be the character of a representation. Then the representation is irreducible if and only if $$\frac{1}{|H|} \sum_{h \in H} |\chi(h)|^2=1.$$
    \end{enumerate}
\end{theorem}
A proof can be found in \citep[p.~16, 17]{serre_RT}.

The following statement collects two further properties of the isotypic decomposition and its trivial component $U_1$.
\begin{theorem} \label{facts:candecomp}
\begin{enumerate}
    \item[a)] $U_1$ is precisely the subspace of $U$ that is fixed under the action of $H$.
    \item[b)] Every decomposition of $U$ into $H$-invariant subspaces refines the canonical decomposition.
\end{enumerate}
\end{theorem}
\begin{proof}
For part (a), we denote the subspace of $U$ which is fixed under $H$ by $$U^H:=\{u \in U: \rho(h) u = u \quad \forall h \in H\}.$$
We will show two inclusions between $U_1$ and $U^H$. Consider $u \in U^H$. Then by \cref{eq:avgop}, $$p_{\text{avg}} (u) = \frac{1}{|H|} \sum_{h \in H} \rho(h) u = \frac{1}{|H|} \sum_{h \in H} u = u.$$
Hence, $u \in \im(p_{\text{avg}})=U_1$.

Conversely, take $u \in U_1$. Then for any $h \in H$ we have $$\rho(h) u = \rho(h) p_\text{avg} u = \frac{1}{|H|} \rho(h)  \sum_{t \in H} \rho(t) u = \frac{1}{|H|}  \sum_{t \in H} \rho(ht) u =p_\text{avg} u=u.$$
Here we used that multiplication by $h$ defines a bijection $H \to H$. We conclude that $u \in U^H$.

Part (b) follows from the uniqueness of the canonical decomposition, which is part of \cref{thm:isotypic}. It implies that we can also obtain the canonical decomposition starting from the given $H$-invariant decomposition. In other words, combining the isomorphic irreducible representations in any given $H$-invariant decomposition yields the canonical decomposition.
\end{proof}

We will now provide proofs for all the results from the main part of this paper, except for Theorems \ref{thm:maschke} and \ref{thm:isotypic}, for which we refer to \cite{serre_RT} p.~7 and p.~21 respectively. 

\subsubsection{Proof of \cref{thm:factorization}}
\begin{proof}
    We can define a map $f':U_1 \to U'$ using the fact that $U_1 \subset U$ by $f'(w):=f(w)$ for all $w \in U_1$. Now take any $u \in U$. Because $\im(p_\text{avg})=U_1$, we have $$f' \circ p_{\text{avg}}(u)=f \circ p_{\text{avg}}(u) = \frac{1}{|H|} \sum_{h \in H} f(\rho(h) u)=\frac{1}{|H|} \sum_{h \in H} f(u)=f(u),$$
    where we used \cref{eq:avgop} and the $H$-invariance of $f$. Hence, our map $f'$ satisfies indeed $f=f' \circ p_\text{avg}$.
\end{proof}

\subsubsection{Proof of Lemma \ref{lem:S_com_action}}
\begin{proof}
    We want to show that for all $h \in H$, we have $S\circ \rho(h)=\rho(h) \circ S$. By the definition of $S$ in \cref{eq:S-orbitals}, it suffices to prove that $\rho(h)$ commutes with all matrices of the form $(M_t+M_t^\top)$.
    
    Note that if $M_t$ is the matrix belonging to the orbit of some element $(i,j) \in [n] \times [n]$, then $M_t^\top$ is the indicator matrix of the orbit of $(j,i)$. Therefore it remains to show that $$M \circ \rho(h)=\rho(h) \circ M$$ holds for all $h \in H$ and orbit indicator matrices $M$. This is a known result in permutation group theory \citep[see e.g.,][p.~37]{cameron}, for which we will provide a short proof here.
    
    We compute the entry $$\bigl(M \rho(h)\bigr)_{ij}=\sum_{k=1}^n M_{ik} \bigl(\rho(h)\bigr)_{kj} = M_{ih^{-1}(j)},$$
    because $\bigl(\rho(h)\bigr)_{kj}=1$ if and only if $h(k)=j$ (and zero otherwise). Similarly,
    $$\bigl(\rho(h)M\bigr)_{ij}=\sum_{k=1}^n \bigl(\rho(h)\bigr)_{ik} M_{kj} = M_{h(i)j}.$$
    Since $h \cdot (i,h^{-1}(j))=(h(i),j)$, these pairs are in the same $H$-orbit of $[n] \times [n]$ and therefore $M_{ih^{-1}(j)}=M_{h(i)j}$.
\end{proof}

\subsubsection{Proof of Lemma \ref{lem:eigenspaces_invariant}}
\begin{proof}
    The orthogonal matrix $Q$ consists of eigenvectors of $S$. Let the columns of $Q$ be the vectors $q_1,\dots,q_n$. They form an eigenbasis of $S$, which is an orthonormal basis, because $Q$ is orthogonal. Take any vector $u$ and write it in the eigenbasis as $u=\sum_{k=1}^n \gamma_k q_k$ for some $\gamma_k \in \IR$. Then from \cref{eq:eigenspace-projector} we see that $P_\alpha u = \sum_{k \in I_\alpha} \gamma_k q_k$. Hence, $P_\alpha$ projects onto the eigenspace $E_\alpha$ belonging to $I_\alpha$: $$E_\alpha=\spn\{q_k:k\in I_\alpha\}=P_\alpha(U).$$
    The orthogonality of the eigenbasis yields orthogonality of the eigenspaces and because the eigenbasis spans $U$, $\{E_\alpha=P_\alpha(U): \alpha = 1,\dots,q\}$ yields an orthogonal decomposition of $U$.

    It remains to show the $H$-invariance of the eigenspaces $E_\alpha$. Take any $h \in H$ and $u \in E_\alpha$. Then we need to show that $\rho(h)u \in E_\alpha$. Let $\lambda$ be the eigenvalue corresponding to $E_\alpha$. Then we have by Lemma \ref{lem:S_com_action}: $$S \rho(h) u = \rho(h) Su = \lambda \rho(h) u.$$
    Hence, also $\rho(h) u$ is an $\lambda$ eigenvector of $S$ and therefore $\rho(h) u \in E_\alpha$, making this space $H$-invariant.
\end{proof}

\begin{lemma} \label{lem:irrep_decompositions}
    Consider the action of the subgroups $H^{(1)}=D_6$ and $H^{(2)}=C_2\times D_3 \times D_3$ of $S_6$. Their action on $\{1,2,3,4,5,6\}$ can be interpreted as permuting the elements of a basis of $U=\IR^6$ yielding two representations of degree 6. These representations have the following canonical decompositions:
    \begin{align*}
        U^{(1)}=&\langle(1,1,1,1,1,1)^\top \rangle \oplus \langle(1,-1,1,-1,1,-1)^\top \rangle
        \oplus \langle (1,1,0,-1,-1,0)^\top,(0,1,1,0,-1,-1)^\top \rangle\\
        &\oplus \langle (1,-1,0,1,-1,0)^\top,(0,1,-1,0,1,-1)^\top \rangle \\
        U^{(2)}=&\langle(1,1,1,1,1,1)^\top \rangle \oplus \langle(1,1,1,-1,-1,-1)^\top \rangle \\ 
        &\oplus \langle (1,-1,0,0,0,0)^\top,0,1,-1,0,0,0)^\top,(0,0,0,1,-1,0)^\top,(0,0,0,0,1,-1)^\top  \rangle.
    \end{align*}
\end{lemma}
\begin{proof}
In this proof, we will work with $U=\IC^6$ instead; by the discussion in Appendix \ref{app:realvscomplex} the results carry over to the real part.

We start with $H^{(1)}$ and first consider the trivial irreducible representation. The projection on the trivial component if given by \cref{eq:avgop}. For $H^{(1)}$ this projection reads
$$p^{(1)}_\text{avg}=\frac 1 {12} \left(\sum_{h \, \in \, \text{rotations}} \rho(h) + \sum_{h \, \in \, \text{reflection and rotations}} \rho(h) \right)=\frac{1}{12}(\mathbf{1}^{6 \times 6}+\mathbf{1}^{6 \times 6})=\frac 1 6 \mathbf{1}^{6 \times 6}.$$
Here $\mathbf{1}^{6 \times 6}$ denotes the $6$ by $6$ matrix with all entries 1. With rotations we mean the elements permuting basis vector indices according to $1 \mapsto 2 \mapsto 3 \mapsto 4 \mapsto 5 \mapsto 6 \mapsto 1$ and iterations thereof.\\
Hence, $p^{(1)}_\text{avg}$ is indeed the projection onto $\mathrm{im}(p^{(1)}_\text{avg}) =\langle(1,1,1,1,1,1)^\top\rangle$.

The remaining three spaces in the proposed canonical decomposition can be verified to also be invariant under the action of $D_6$. Alternatively, Lemma \ref{lem:eigenspaces_invariant} together with the fact that these spaces are the eigenspaces of $S^{(1)}$ in Example \ref{ex:2} also show the invariance of these spaces.

Hence it only remains to check that they belong to different irreducible representations. Note that $\langle(1,-1,1,-1,1,-1)^\top \rangle$ must belong to its own irreducible representation, since it is an irreducible representation of degree 1 that is not the trivial one.

\citet[p.~37]{serre_RT} computes the characters of the two degree 2 irreducible representations of $D_6$. They are given by $\chi_j(r^k)=2 \cos\left(\frac{\pi jk}{3}\right)$, where $r$ denotes the rotation corresponding to $1 \mapsto 2 \mapsto 3 \mapsto 4 \mapsto 5 \mapsto 6 \mapsto 1$, and $\chi_j=0$ for compositions of reflection and rotation for $j \in \{1,2\}$. By \cref{eq:character_proj}, the projections onto the corresponding component are given by
$$p_{\chi_1}=\frac 1 3 \begin{pNiceMatrix}
    1& \frac 1 2 & -\frac 1 2 & -1 & -\frac 1 2 & \frac 1 2 \\
    \frac 1 2 & 1& \frac 1 2 & -\frac 1 2 & -1 & -\frac 1 2 \\
    -\frac 1 2 & \frac 1 2& 1& \frac 1 2 & -\frac 1 2 & -1 \\
    -1 & -\frac 1 2 & \frac 1 2 & 1& \frac 1 2 & -\frac 1 2 \\
    -\frac 1 2 & -1 & -\frac 1 2 & \frac 1 2& 1& \frac 1 2   \\
    \frac 1 2 & -\frac 1 2 & -1 & -\frac 1 2 & \frac 1 2& 1  \\
\end{pNiceMatrix} \quad \text{and} \quad p_{\chi_2}=\frac 1 3 \begin{pNiceMatrix}
    1& -\frac 1 2 & -\frac 1 2 & 1 & -\frac 1 2 & -\frac 1 2 \\
    -\frac 1 2 & 1& -\frac 1 2 & -\frac 1 2 & 1 & -\frac 1 2 \\
    -\frac 1 2 & -\frac 1 2& 1& -\frac 1 2 & -\frac 1 2 & 1 \\
    1 & -\frac 1 2 & -\frac 1 2 & 1& -\frac 1 2 & -\frac 1 2 \\
    -\frac 1 2 & 1 & -\frac 1 2 & -\frac 1 2& 1& -\frac 1 2   \\
    -\frac 1 2 & -\frac 1 2 & 1 & -\frac 1 2 & -\frac 1 2& 1  \\
\end{pNiceMatrix}.$$
These project precisely onto $\langle (1,1,0,-1,-1,0)^\top,(0,1,1,0,-1,-1)^\top \rangle$ and $ \langle (1,-1,0,1,-1,0)^\top,(0,1,-1,0,1,-1)^\top \rangle$ respectively, concluding our proof for $H^{(1)}$.

With a similar argumentation, for $H^{(2)}$ it only remains to check that the two one-dimensional irreducible representations are not isomorphic. We will do that again by showing that $\mathrm{im}(p^{(2)}_\text{avg}) =\langle(1,1,1,1,1,1)^\top\rangle$. Every element of $H^{(2)}$ can be written as applying any element from $D_3=S_3$ to the vertices 1,2 and 3 in \cref{fig:ex1}, applying any element from $S_3$ to 4,5 and 6 and then either exchanging $1,2,3$ with $4,5,6$ or not. Hence we get
\begin{align*}
    p^{(2)}_\text{avg} &= \frac 1 {72} \left(\left(\sum_{h \in \text{permutations of } 1,2,3 } \rho(h) \right) \left(  \sum_{h \in \text{permutations of } 4,5,6 } \rho(h) \right) \bigl(\rho(\text{id})+\rho((1,2,3)\leftrightarrow (4,5,6) )\bigr) \right)\\
    &=\begin{pNiceMatrix}
        2&2&2&0&0&0\\
        2&2&2&0&0&0\\
        2&2&2&0&0&0\\
        0&0&0&6&0&0\\
        0&0&0&0&6&0\\
        0&0&0&0&0&6\\
    \end{pNiceMatrix} \begin{pNiceMatrix}
        6&0&0&0&0&0\\
        0&6&0&0&0&0\\
        0&0&6&0&0&0\\
        0&0&0&2&2&2\\
        0&0&0&2&2&2\\
        0&0&0&2&2&2\\
    \end{pNiceMatrix} \begin{pNiceMatrix}
        1&0&0&1&0&0\\
        0&1&0&0&1&0\\
        0&0&1&0&0&1\\
        1&0&0&1&0&0\\
        0&1&0&0&1&0\\
        0&0&1&0&0&1\\
    \end{pNiceMatrix}=\frac 1 6 \mathbf{1}^{6 \times 6}.
\end{align*}
\end{proof}

\begin{lemma} \label{lem:welldef}
    $R_{\mathrm{iso}}(G,M)$ does not change under node relabeling.
\end{lemma}
\begin{proof}
    Let $\Pi \in \IR^{n \times n}$ be the matrix representation of a node relabeling, so the matrix belonging to a permutation $\pi$ of vertex indices: if $\Pi_{ij}=1$ this means that the previous vertex $j$ is now called $i$. Note that we have $\Pi^\top = \Pi^{-1}$.\\
    Let $G^\Pi$ be the relabeled graph. Then the adjacency matrix transforms as $A(G^\Pi)=\Pi A(G)\Pi^\top$ and node embedding transforms as $M^\Pi=\Pi M$.

    Understanding elements of $\mathrm{Aut}(G)$ as permutation matrices too, we have $\mathrm{Aut}(G^\Pi)=\Pi\,\mathrm{Aut}(G)\,\Pi^\top$. Therefore, $H$-orbits under the induced action on pairs are conjugated by $\Pi$ and consequently, $M_t(G^\Pi)=\Pi M_t(G)\Pi^\top$ for each orbit indicator matrix. Hence, $S(G^\Pi)=\Pi S(G)\Pi^\top$.
    Thus, we obtain the eigendecomposition $S(G^\Pi)=\Pi Q \Lambda Q^\top \Pi^\top$ with $Q$ and $\Lambda$ taken from the eigendecomposition of $S$. Since the three sorting criteria do not change (traces of similar matrices are equal), we obtain
    $\{P_\alpha(G^\Pi)\}_{\alpha=1}^B = \{\Pi P_\alpha(G)\Pi^\top\}_{\alpha=1}^B$.

    Therefore, we get $M_\alpha^\Pi=P_\alpha^{\Pi} M^\Pi=\Pi M_\alpha$, i.e., the rows are permuted. All three summaries $s_j$ and also the column average $\mu$ involve summation over all row indices and are therefore invariant under row permutations. Because the matrix $R$ is fix, this implies $\Psi_\alpha(G^\Pi,M^\Pi)=\Psi_\alpha(G,M)$ and consequently, $$R_\text{iso}(G^\Pi,M^\Pi) = R_\text{iso}(G,M).$$
\end{proof}

\section{Method: Centering}
\label{app:centering}
In this section of the Appendix we describe a preprocessing step that can optionally be included at the beginning of our readout, and replaces the matrix $M$ by a matrix $M_c$ whose columns sum to zero.

The analysis in \cref{sec:bottleneck} shows that linear, fully permutation-invariant readouts can access only the column-averaged components of the node embedding matrix $M\in\IR^{n\times d}$, which are fixed under row permutations. Let $\mathbf{1}^{n \times n}$ be the $n \times n$ matrix with all-1 entries. Then sum and mean pooling can only see
\[
M_{\mathrm{avg}}:=\frac{1}{n} \mathbf{1}^{n \times n} M.
\]
This observation motivates separating the constant (fully invariant) part of the node embeddings from the variation across nodes. To do this, we define the centering operator
\begin{equation*}
J \ :=\ I_n - \frac{1}{n}\mathbf{1}^{n \times n},
\quad
M_c \ :=\ JM.
\end{equation*}
Then every node embedding matrix admits the decomposition $M=M_{\mathrm{avg}}\ +\ M_c$ into an average part and a centered part with column average zero. $M_c$ has zero column-wise mean and captures exactly the node-to-node variation that is removed by global averaging.

Centering therefore isolates the components of the representation that are invisible to standard sum or mean pooling, but potentially informative for downstream tasks.  For this reason, we allow centering as an optional preprocessing step in our readout architecture.  When enabled, we apply the block projectors to the centered embeddings $M_\alpha:=P_\alpha J M$ instead of $M_\alpha=P_\alpha M$.  In practice, centering can improve numerical stability and encourages the readout to focus on relative differences between node representations rather than global offsets. Importantly, centering preserves permutation equivariance and does not introduce any additional learnable parameters.

\section{Additional Experiments: Implementation, Results, and Ablations}
\label{app:reproducibility}
This part of the appendix provides more details about the performed experiments. We start by explaining the concrete implementation, before we present additional numerical results and some ablation studies.

\subsection{Implementation Details}

Here, we provide full implementation details and computing infrastructure used for our experiments.

\paragraph{Software and Determinism.}  All experiments use Python (3.10+), PyTorch (\texttt{torch>=2.3}) and PyTorch Geometric (\texttt{torch-geometric>=2.5}), with \texttt{ogb}, \texttt{networkx}, \texttt{scipy}, and \texttt{scikit-learn}. Training uses single precision (float32) throughout (no AMP). We control randomness with a single run seed that initializes Python, NumPy, and PyTorch (CPU and CUDA), and we run cuDNN in deterministic mode (\texttt{cudnn.deterministic=True}, \texttt{cudnn.benchmark=False}). Data shuffling uses PyTorch RNG under the same seed (default \texttt{num\_workers=0}).

\subsubsection{Architectures}
All models follow a shared scaffold: a message-passing \emph{encoder} producing node embeddings in $\mathbb{R}^{n\times d}$, a \emph{readout} mapping them to a graph embedding in $\mathbb{R}^{d_r}$, and a 2-layer MLP head. The head performs a linear map $\IR^{d_r}\!\to\! \IR^{d_r}$, followed by ReLU, dropout (same probability as the encoder), and another linear map $\IR^{d_r}\!\to\!$ task output.\\ Unless stated otherwise below, we apply BatchNorm after each message-passing layer. The activations we use are ReLU (and ELU for GATv2).

\paragraph{Encoders.}  As encoders we use standard PyG implementations: GIN, PNA, GraphSAGE, GATv2, and a TransformerConv-based graph transformer. For PNA, we compute the degree histogram on the \emph{training split only} and pass it to PNAConv (aggregators \{\texttt{mean,min,max,std}\}, scalers \{\texttt{identity,amplification,attenuation}\}, and \texttt{towers=pre\_layers=post\_layers=1}). The graph transformer uses 6 TransformerConv layers with \texttt{beta=True}. Here are the concrete encoder settings used for the different tasks:
\begin{itemize}[noitemsep]
    \item \textbf{Expressivity tasks (\cref{subsec:trainfree_sep,subsec:srg16}):}
GIN: $d{=}64$, 5 layers, dropout 0.0;\\
PNA/GraphSAGE/GATv2: $d{=}96$, 5 layers, dropout 0.0 (GATv2 uses 4 heads);\\
graph transformer: $d{=}128$, 6 layers, 8 heads, dropout 0.1.
\item \textbf{Downstream tasks (\cref{subsec:downstream}):}
GIN/PNA/GINE: $d{=}128$, 5 layers, dropout 0.1 (except the readout-swap experiment below, which uses dropout 0.0).\\
For MolHIV, GINE uses the OGB atom/bond encoders; otherwise it falls back to linear encoders over float features.
\end{itemize}

\paragraph{Classical Readouts.} Baselines are global sum, mean, max, mean+max concatenation, PNA-style statistics pooling (mean/sum/max/std concatenation), global attention pooling (a 2-layer gate MLP with hidden size 128 and scalar output), and Set2Set with 3 processing steps. We also use \emph{combo} readouts that fuse multiple readouts either by concatenation, by summation after projection to a shared fusion dimension, or by a learned softmax gate per graph. For ``gated'' fusion the fusion dimension, unless otherwise specified, defaults to $d$.

\paragraph{Isotypic Readouts.}
Our isotypic readout computes invariant block summaries from node embeddings using a set of orthogonal projectors $\{P_\alpha\}$ (``blocks''). Given a node feature matrix $M$, for each block, we form $P_\alpha M$ (optionally after centering) and produce features consisting of:
\begin{enumerate}[noitemsep]
    \item three norms (of the summed vector, Frobenius norm, and mean row norm), and
    \item a fixed random projection of the block mean.
\end{enumerate}
The random projection matrix $R$ has shape $d\times r$ with $r{=}8$ and is sampled once at initialization (and is therefore deterministic given the run seed). The final graph embedding concatenates features from at most \texttt{max\_blocks} blocks (padding with zeros when fewer exist).

Projectors are constructed by \emph{orbital blocks} obtained by enumerating automorphisms via NetworkX GraphMatcher up to \texttt{cap\_auts}=50{,}000, constructing orbit indicator matrices, forming a random symmetric combination, and taking eigenspace projectors, using tolerance $10^{-12}$ to tell eigenvalues apart. 

\paragraph{Optimization and Hyperparameters.}  All supervised trainings use Adam (no scheduler, no gradient clipping) and early stopping on the validation metric. The best checkpoint (by validation score) is restored before testing.
Optimization happens with respect to the following loss functions:
\begin{itemize}[noitemsep]
    \item for multi-class graph classification: cross-entropy / accuracy;
    \item for MolHIV: \texttt{BCEWithLogitsLoss} / ROC-AUC (masking OGB's missing labels);
    \item for regression: MSE / MAE (we also log $R^2$).
\end{itemize}
Table~\ref{tab:opt-hparams} lists the exact optimization settings used.

\begin{table}[t]
\centering
\small
\setlength{\tabcolsep}{5pt}
\begin{tabular}{lccccc}
\hline
\textbf{Setting} & \textbf{Batch} & \textbf{Epochs} & \textbf{Learning Rate} & \textbf{Weight Decay} & \textbf{Patience} \\
\hline
SRG16 (WL-hard)  & 64  & 300 & $10^{-3}$ & $10^{-5}$ & 40 \\
BREC RPC-lite    & 64  & 200 & $10^{-3}$ & $10^{-5}$ & 30 \\
ZINC (PyG subset) & 128 & 300 & $10^{-3}$ & $10^{-5}$ & 40 \\
MolHIV (OGB split) & 256 & 200 & $10^{-3}$ & $10^{-5}$ & 30 \\
Spectral (ENZYMES/PROTEINS) & 32 & 300 & $10^{-3}$ & $10^{-5}$ & 40 \\
\hline
\end{tabular}
\caption{Optimization Hyperparameters. }
\label{tab:opt-hparams}
\end{table}

\subsubsection{Dataset Protocols}
We now provide details of the datasets used in the different experiments.

\paragraph{Training-Free Separation.}  We evaluate 36 WL-hard graph pairs. For each seed and each pair of encoder and readout, we initialize the model once and compute graph embeddings on a random node permutation of each graph. We record cosine similarity per pair and count a pair as separated if:
\begin{enumerate}[noitemsep]
    \item the mean similarity across seeds is below the set threshold, and
    \item a one-sided one-sample $t$-test against 1.0 is significant after Holm--Bonferroni correction at $\alpha=0.05$.
\end{enumerate}

\paragraph{WL-Hard Supervised.}  For SRG-16 we load the provided graph6 collection of Shrikhande and rook graphs, sample the configured number of graphs per class (5{,}000 each), apply consistent node permutations, and split 80/10/10 with the run seed. Node features are constant ones.

\paragraph{BREC RPC-lite.}  For each of the 400 BREC pairs, we build three disjoint datasets of permuted copies: 64 permutations per base graph for training, 32 for validation, and 128 for testing.  Permutations are generated with a seeded \texttt{torch.Generator}.  Each permuted graph stores both the permutation and its inverse; for the isotypic readout, this enables computing orbit projectors once in a fixed base labeling and reusing them across relabelings.  A pair is counted as solved if the test accuracy exceeds $0.95$; for reporting ``solved'', we require this threshold to be met for all seeds in \{0,1,2\}.

\paragraph{Downstream Benchmarks.}  ZINC is loaded via PyG's \emph{subset} split, using the provided train/val/test partitions; MolHIV uses the official OGB split. 
On ENZYMES and PROTEINS, the target is the vector of the $k=8$ smallest nonzero eigenvalues of the normalized Laplacian. We compute eigenvalues with \texttt{scipy.sparse.linalg.eigsh}, and split graphs 80/10/10 with the run seed.

\subsubsection{Hardware and SLURM Settings}
All runs were executed on a SLURM-managed GPU cluster. Each job requests a single GPU (\texttt{--gres=gpu:1}) on the \texttt{t4} partition, with 2--6 CPU cores and 8--32\,GB RAM depending on the experiment. This is sufficient to reproduce all reported results with the stated seeds and hyperparameters. All experiments reported in the paper used a single NVIDIA Tesla T4 GPU (16\,GB VRAM) per job (Slurm partition \texttt{t4}, \texttt{--gres=gpu:1}).

\subsection{Runtime and Scalability}
\label{app:runtime}

The isotypic readout replaces standard global pooling by constructing graph-dependent projectors $\{P_\alpha\}$ from orbit structure and computing block-wise invariant statistics. A key practical distinction from sum and mean pooling is that the cost is not governed by graph size alone. On graphs with rich symmetry, the orbit routines can enter a substantially heavier regime. 

\subsubsection{End-to-End Complexity (Encoder and Readout)} 
We start by giving a theoretical estimate of the total time complexity. Let $G=(V,E)$ be a graph with $n=|V|$, $m=|E|$, hidden width $d$, and $L$ message-passing layers. A standard sparse message-passing encoder costs
\[
T_{\text{enc}}(G) \;=\; O\!\left(L \cdot m \cdot d\right)
\quad\text{(up to architecture-dependent constants).}
\]
Our readout computes a graph-dependent projector construction followed by block-wise invariant statistics. In the implementation used in this paper, projector construction proceeds by:
\begin{enumerate}[noitemsep]
    \item sampling up to $A$ automorphisms (\texttt{cap\_auts}$=A$),
    \item assembling a dense equivariant operator $S\in\IR^{n\times n}$ from orbit indicators, and
    \item computing an eigendecomposition of $S$ to obtain the block projectors.
\end{enumerate}
This yields the per-graph bound
\[
T_{\text{read}}(G)
\;=\;
T_{\\Aut}(G;A)
\;+\;
O(A\,n^2)
\;+\;
O(n^3)
\;+\;
O(n^2 d),
\]
where $T_{\Aut}(G;A)$ is the time spent in the orbit routine under the sampling cap, $O(A n^2)$ is orbit-based operator assembly, $O(n^3)$ is a dense eigendecomposition upper bound for the current implementation, and $O(n^2 d)$ accounts for applying projectors (equivalently, working through the eigenbasis) and computing block-wise statistics.
Which of these terms \emph{dominates} depends on symmetry: on benign graphs the runtime is governed by the eigendecomposition, while on symmetry-heavy graphs, computing automorphisms and orbits can dominate and produce heavy-tail behavior.

\subsubsection{Runtime Experiments}

\paragraph{Controlled Scaling on Random Graphs (Erd\H{o}s--R\'enyi).}
\label{app:runtime_er}
To obtain a benign baseline where automorphisms are typically trivial, we measure runtime on Erd\H{o}s--R\'enyi (ER) graphs $G(n,p)$. These are random graphs on $n$ points where each edge is present with probability $p$. We use $p=0.1$ and $n\in\{16,32,48,64,80,96,112,128,160,192,224\}$. For each $n$ we sample 50 graphs and report median and 90th-percentile readout time.
\begin{figure}[h]
\centering
\includegraphics[width=0.72\linewidth]{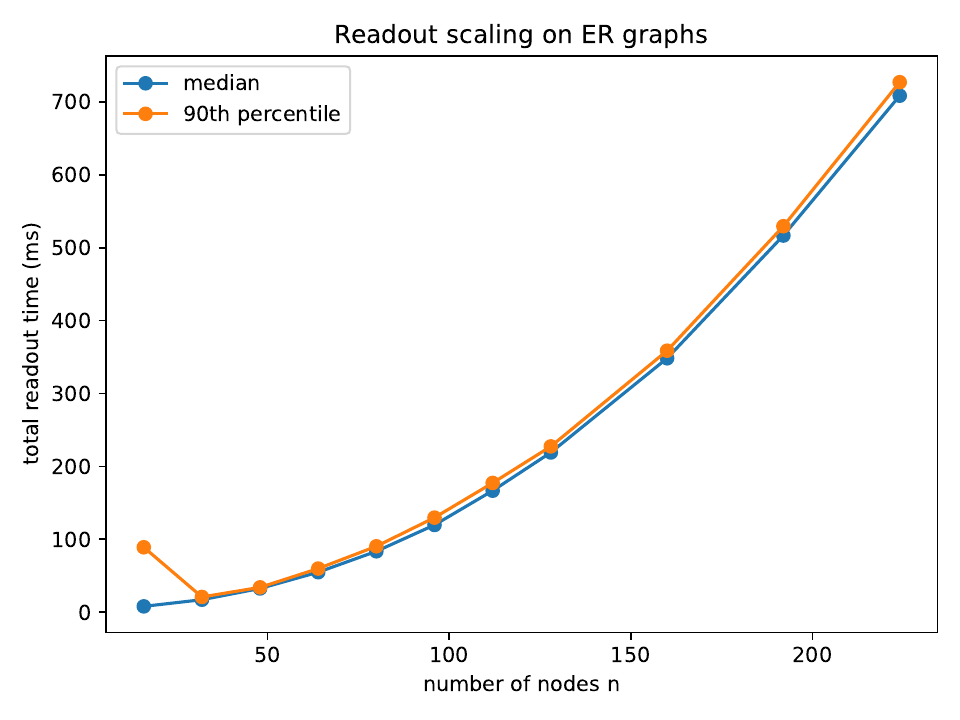}
\caption{Readout scaling on ER graphs ($p=0.1$): Total readout time vs.\ number of nodes (median and 90th percentile over 50 graphs).}
\label{fig:runtime_er_scaling}
\end{figure}

We observe that runtime increases continuously with $n$ and ``explosive'' behavior is rare, because median and 90th percentile are similar, see \cref{fig:runtime_er_scaling}. This provides a baseline for datasets where symmetry is not adversarially constructed: the overhead is predictable and dominated by the eigendecomposition stage rather than by orbit enumeration.

\paragraph{Readout Runtimes for BREC.}
In \cref{app:runtime_brec} we evaluate our readout for the official BREC dataset and keep track of runtimes. Graphs used in BREC are usually highly symmetric, so have huge symmetry groups. BREC-style benchmarks are organized as \emph{pairs} of non-isomorphic graphs.
Accordingly, we report runtimes both \emph{per graph} and \emph{per pair} (two graphs, one-shot readout, no training); the latter is the relevant unit for explaining wall-clock feasibility.

We record per-graph: 
\begin{enumerate}[noitemsep]
    \item projector construction time,
    \item block-wise feature extraction time given node embeddings, and
    \item total readout time (sum of the two).
\end{enumerate}
These timings isolate \emph{readout overhead} and do \emph{not} include message passing, task heads, or backpropagation. To avoid unbounded enumeration, automorphism sampling is capped at \texttt{cap\_auts}$=1000$; we additionally report the fraction of graphs/pairs that reach this cap as a proxy for difficult symmetry structure.
\subsection{Further Experimental Results}
\label{appsub:results}
We report full tabulated results mentioned in the main text in Section \ref{sec:experiments}.

\begin{table}[h]
\centering
\caption{Training-free separation suite (36 WL-equivalent instances). A pair is \emph{separated} if the mean cosine similarity is below $0.95$.
We report mean/max cosine similarity across the 36 instances (lower is better).}
\label{tab:trainfree_main}
\small
\begin{tabular}{llccc}
\toprule
\textbf{Encoder} & \textbf{Readout} & \textbf{Separated} & \textbf{Mean cos} & \textbf{Max cos} \\
\midrule
GIN       & Baseline       & 0/36   & 1.000 & 1.000 \\
GraphSAGE & Baseline       & 0/36   & 1.000 & 1.000 \\
PNA       & Baseline       & 0/36   & 1.000 & 1.000 \\
\midrule
GIN       & isotypic  & 33/36  & 0.081 & 0.973 \\
GraphSAGE & isotypic  & 33/36  & 0.106 & 1.000 \\
PNA       & isotypic  & 33/36  & 0.067 & 1.000 \\
\bottomrule
\end{tabular}
\end{table}

\begin{table}[H]
\centering
\caption{Family breakdown (GIN) for the training-free separation suite.
We report the mean cosine similarity between the two graph embeddings (lower is better).}
\label{tab:tf_family_gin}
\small
\setlength{\tabcolsep}{3.5pt}
\begin{tabular}{lcrr}
\toprule
\textbf{Family} & \textbf{\#} & \textbf{Baseline} & \textbf{Isotypic} \\
\midrule
Cycles ($2C_k$ vs.\ $C_{2k}$) & 24 & 1.000 & 0.000 \\
CFI-$K_3$                    & 3  & 1.000 & 0.000 \\
CFI-$K_4$                    & 3  & 1.000 & 0.000 \\
Helical ribbons              & 3  & 1.000 & 0.973 \\
GM--Petersen                 & 3  & 1.000 & 0.013 \\
\bottomrule
\end{tabular}
\vspace{-2mm}
\end{table}

\begin{table}[H]
\centering
\caption{BREC RPC-lite average test accuracy (GIN). Chance is $0.5$.}
\label{tab:brec_rpclite_main}
\small
\begin{tabular}{lcc}
\toprule
\textbf{Split} & \textbf{Multiset pooling} & \textbf{Isotypic readout} \\
\midrule
BREC dataset     & 0.50 & 0.901 \\
\bottomrule
\end{tabular}
\end{table}

\begin{table}[H]
\centering
\caption{SRG-16 test accuracy across encoders. All multiset-based baselines
(\texttt{sum/mean/meanmax/Set2Set/attn/pna}) achieve chance-level performance ($\approx 0.45$),
while the isotypic readout reaches $1.00$ across all encoders.}
\label{tab:srg16_main}
\small
\begin{tabular}{lcc}
\toprule
\textbf{Encoder} & \textbf{Baseline readouts} & \textbf{Isotypic readout} \\
\midrule
GIN               & 0.45 & 1.00 \\
GraphSAGE         & 0.45 & 1.00 \\
GATv2             & 0.45 & 1.00 \\
PNA               & 0.45 & 1.00 \\
Graph Transformer & 0.45 & 1.00 \\
\bottomrule
\end{tabular}
\end{table}


\subsection{Official BREC Evaluation (Partial)}
\label{subsec:brec_official_partial}

To connect our findings to the standard BREC evaluation protocol, we ran the official pipeline (which trains a separate model per pair using multiple relabelings and applies the benchmark's fixed decision rule). This experiment is included as a protocol-compatibility check rather than as a full benchmark submission.

\paragraph{Coverage and Results.}
The official BREC evaluation is expensive for two reasons. First, it trains a \emph{separate} model for each pair, so the total cost scales linearly with the number of pairs. Second, and more importantly for our setting, several BREC families contain graphs with very large automorphism groups. On such instances, the underlying symmetry can be enormous, which makes any computation that depends on graph symmetries---including the construction used by our readout---substantially more expensive.
In practice, these highly symmetric pairs dominate wall-clock time and memory, and make running the full 400 pairs prohibitive under a fixed budget. We therefore evaluated only the first 260 pairs (indices 0--259 in the official ordering) and obtained \textbf{56/260} distinguished pairs. A breakdown over the categories solved within this subset is shown in Table~\ref{tab:brec_official_partial}.

\begin{table}[h]
\centering
\caption{Official BREC results on the first 260 pairs (0--259).}
\label{tab:brec_official_partial}
\small
\begin{tabular}{lcc}
\toprule
\textbf{Category} & \textbf{GIN+isotypic} & \textbf{GIN+sum} \\
\midrule
Basic     & 2/60  & 0/60 \\
Regular   & 48/100 & 0/100 \\
Extension & 6/100  & 0/100 \\
\midrule
\textbf{Total} & \textbf{56} & \textbf{260} \\
\bottomrule
\end{tabular}
\end{table}

\paragraph{Interpretation.}
For reference, the original BREC study evaluates 23 beyond-1-WL architectures and reports a best overall distinguishability of $70.2\%$ on the full 400-pair benchmark (achieved by I2-GNN \cite{huang2022boosting}), indicating that BREC remains far from being saturated. However, our goal in this section is not to compete with specialized architectures, but to isolate whether a \emph{readout change alone} can recover distinguishability when the encoder is held fixed. The official benchmark is a regime where standard message-passing baselines with multiset pooling are extremely weak, so improving distinguishability through a \emph{readout swap alone} is already meaningful evidence that the readout is a bottleneck in practice.

We emphasize that our main paired-comparison evidence remains RPC-lite, which directly measures generalization across held-out relabelings and allows controlled sweeps under a fixed budget. 

\subsubsection{Runtimes for Official BREC Evaluation}
\label{app:runtime_brec}

BREC is explicitly designed to contain symmetry-heavy non-isomorphic pairs, making it a natural stress test for our automorphism-based readout. We report runtimes for a BREC subset of 193 pairs (386 graphs).

\paragraph{Results.}  Although the \emph{median} per-graph readout time is small, the distribution is extremely heavy-tailed, see \cref{tab:runtime_brec_summary} and \cref{fig:runtime_brec_tail} (left).
Crucially, this tail is not explained by node count alone: graphs of similar size can differ by orders of magnitude in runtime (see the scatter plot of \cref{fig:runtime_brec_tail}), verifying that automorphism structure dominates the runtime for highly symmetric graphs.

\begin{figure}[h]
\centering
\begin{minipage}{0.49\linewidth}
  \centering
  \includegraphics[width=\linewidth]{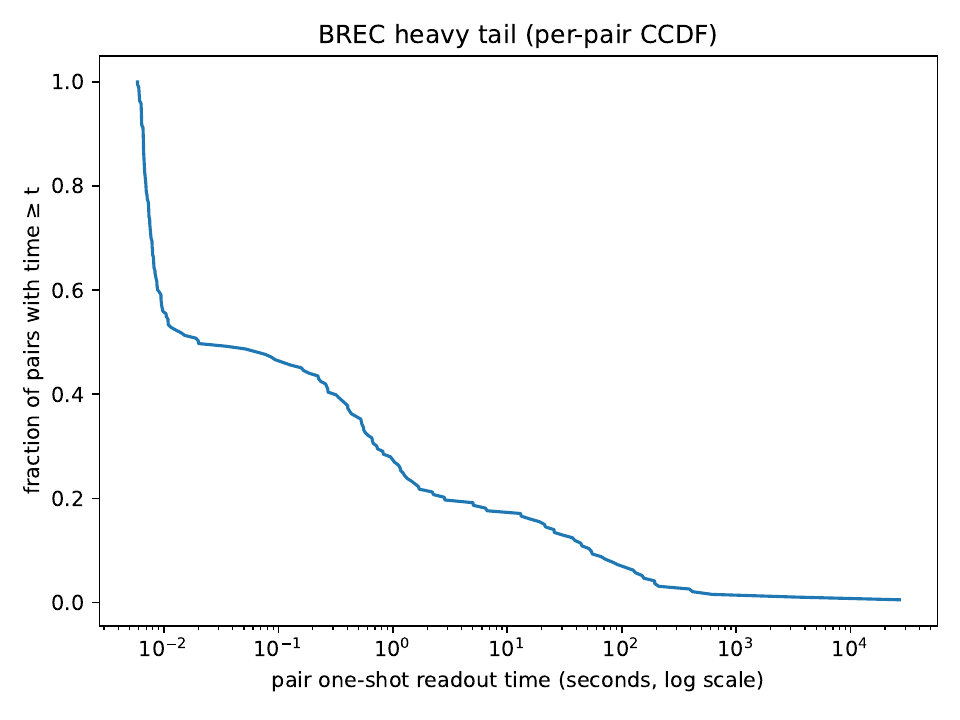}
\end{minipage}\hfill
\begin{minipage}{0.49\linewidth}
  \centering
  \includegraphics[width=\linewidth]{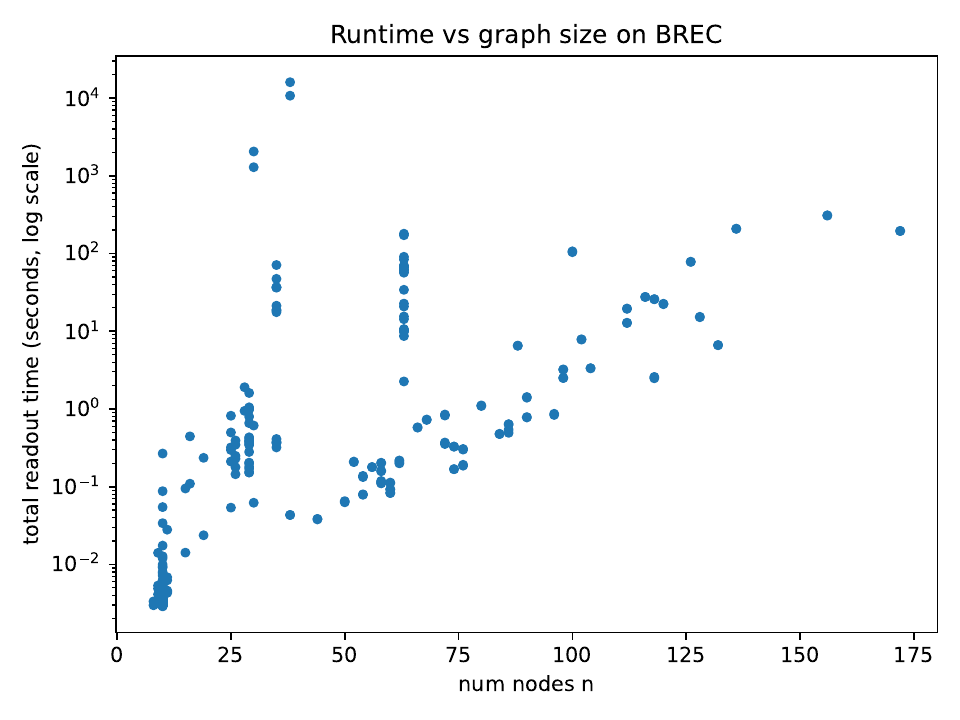}
\end{minipage}
\caption{\textbf{Left:} Complementary Cumulative Distribution Function (CCDF) of per-pair readout time on the 193 measured BREC pairs.
\textbf{Right:} Scatter of node count $n$ vs.\ readout time.}
\label{fig:runtime_brec_tail}
\end{figure}

\begin{table}[H]
\centering
\caption{BREC readout heavy-tail summary (193 pairs / 386 graphs). Timings are readout-only.}
\label{tab:runtime_brec_summary}
\small
\begin{tabular}{lcccc}
\toprule
 & \textbf{Median} & \textbf{$p90$} & \textbf{$p95$} & \textbf{$p99$} \\
\midrule
Per-graph total time (s) & 0.0079 & 21.0 & 69.7 & 456 \\
Per-pair total time (s)  & 0.020  & 50.3 & 139  & 834 \\
\bottomrule
\end{tabular}

\vspace{4pt}
\small
\begin{tabular}{lccccc}
\toprule
 & $\ge 0.1$s & $\ge 1$s & $\ge 10$s & $\ge 60$s & $\ge 600$s \\
\midrule
Graphs (fraction) & 42.5\% & 20.2\% & 14.0\% & 6.2\% & 1.0\% \\
Pairs (fraction)  & 46.1\% & 26.9\% & 17.1\% & 8.8\% & 1.6\% \\
\bottomrule
\end{tabular}
\end{table}

\paragraph{From Per-Pair Readout Time to Official BREC Wall-Clock.}
\label{app:runtime_official_cost}

The official BREC evaluation performs \emph{per-pair training} and uses many relabelings per graph. A simple lower-bound cost model for the readout stage alone is
\begin{equation}
T_{\text{official}}(\text{pair}) \;\approx\; E \cdot R \cdot t(G),
\label{eq:official_cost_model}
\end{equation}
where $t(G)$ is the measured per-pair readout time, $R$ is the number of relabelings per graph (typically $R=32$), and $E$ is the number of optimization epochs. Plugging in the measured per-pair distribution and taking a representative $E=50$ yields the illustrative readout-only wall-clock estimates in \cref{tab:runtime_brec_official_est} below. Note that this excludes runtimes for message passing, backpropagation, and optimizer overhead, so it is a lower bound.

\begin{table}[H]
\centering
\caption{Illustrative official BREC readout-only cost per pair (hours), using Eq.~\eqref{eq:official_cost_model} with $R=32$ relabelings and $E=50$ epochs.}
\label{tab:runtime_brec_official_est}
\small
\begin{tabular}{lcccc}
\toprule
 & \textbf{Median} & \textbf{$p90$} & \textbf{$p95$} & \textbf{$p99$} \\
\midrule
$T_{\text{official}}(\text{pair})$ (hours) & 0.0089 & 22.3 & 62.0 & 371 \\
\bottomrule
\end{tabular}
\end{table}

We see that even before accounting for encoder forward and backward passes, the heavy tail implies that a small fraction of symmetry-heavy pairs dominates wall-clock time.
This explains why full official coverage becomes prohibitive without additional engineering, such as caching projectors across relabelings, canonicalization strategies, or fallback policies on detected runtime explosions. It also motivates our focus on RPC-lite in the main paper.

RPC-lite and the official BREC benchmark test the same requirement---permutation-invariant but pair-sensitive representations---but they stress computations in very different ways.
RPC-lite trains a \emph{single} classifier per configuration (consisting of encoder, readout and seed) for one permutation of a dataset. So the readout is computed once per sampled graph and amortized over the whole run.
In contrast, the official BREC pipeline performs \emph{per-pair training} with many relabelings and many optimization steps. The readout is recomputed repeatedly for each relabeling. This, together with the computing time of several minutes of readouts for some highly symmetric graphs, makes full official evaluation of BREC impractical without further optimization. 
\subsection{Ablations (GIN).}
\label{app:ablations}
Here we present the results of our ablations studies for the training-free model and for BREC RPC-lite.

\subsubsection{Ablations for Training-Free Model}
\label{app:trainfree_ablations}

We ablate two core design choices of the readout using the same fixed training-free protocol. To avoid conflating statistical power with effect size, we use one seed per experiment and consider the graphs $G,G'$ separated if $\cos\bigl(z(G),z(G')\bigr) < 0.95$. The ablations are designed to isolate whether separation arises from 
\begin{enumerate}[noitemsep]
    \item the graph-dependent projector construction or
    \item the use of nonlinear block-wise statistics.
\end{enumerate}

\begin{table}[h]
\centering
\caption{Training-free ablation (GIN): projector family. Separation uses $\cos(z(G),z(G'))<0.95$.}
\label{tab:trainfree_projfamily}
\small
\begin{tabular}{lcc}
\toprule
\textbf{Readout} & \textbf{Separated} & \textbf{Total} \\
\midrule
isotypic          & 33 & 36 \\
isotypic\_linear  & 3  & 36 \\
sum               & 0  & 36 \\
\bottomrule
\end{tabular}
\end{table}

Table~\ref{tab:trainfree_projfamily} compares the isotypic readout to two baselines: a linearized variant that applies only linear statistics within each block, and standard sum pooling. The linearized readout collapses almost to the pooling baseline, validating the theory of linear permutation-invariant maps, which restricts such readouts to a fixed invariant subspace. The few separations observed are attributable to numerical or finite-sample effects rather than principled discrimination. In contrast, the full isotypic readout separates the vast majority of pairs.

\begin{table}[h]
\centering
\caption{Training-free ablation (GIN): effect of the block budget \texttt{max\_blocks}. Separation uses $\cos(z(G),z(G'))<0.95$.}
\label{tab:trainfree_blocksweep}
\small
\begin{tabular}{lcc}
\toprule
\textbf{max\_blocks} & \textbf{Separated} & \textbf{Total} \\
\midrule
1   & 9  & 36 \\
2   & 23 & 36 \\
4   & 28 & 36 \\
8   & 33 & 36 \\
16  & 33 & 36 \\
\bottomrule
\end{tabular}
\end{table}

Table~\ref{tab:trainfree_blocksweep} varies the block budget $\texttt{max\_blocks}$.  Separation improves monotonically as the number of retained blocks increases, and saturates around $B = 8$. This indicates that most useful invariant structure is captured by a relatively small number of isotypic components, suggesting favorable practical trade-offs. 

\subsubsection{BREC Ablations}   
To verify that the improvement in the RPC-lite experiment \cref{subsec:srg16} comes from the isotypic mechanism rather than from classifier training itself, we ablate the readout family under the same RPC-lite protocol.
Table~\ref{tab:brec_ablation_projfamily} shows that the linearized variant collapses to chance, mirroring the training-free behavior.
\cref{tab:brec_blocksweep} demonstrates that accuracy improves monotonically with \texttt{max\_blocks}, with most gains realized by 8 blocks. This is consistent with the readout extracting increasingly rich graph-dependent invariant structure as more blocks are retained, while multiset pooling remains at chance. It also suggests that a relatively small number of blocks suffices to obtain a significantly improved readout performance.

\begin{table}[H]
\centering
\caption{BREC RPC-lite ablation (GIN): projector family. Average accuracy on the 100-pair subset.}
\label{tab:brec_ablation_projfamily}
\small
\begin{tabular}{lc}
\toprule
\textbf{Readout} & \textbf{Avg test acc} \\
\midrule
isotypic          & 0.905 \\
isotypic\_linear  & 0.500 \\
mean              & 0.500 \\
\bottomrule
\end{tabular}
\end{table}

\begin{table}[H]
\centering
\caption{BREC RPC-lite ablation (GIN): effect of the block budget \texttt{max\_blocks}. Average test accuracy on the 100-pair subset.}
\label{tab:brec_blocksweep}
\small
\begin{tabular}{lcc}
\toprule
\textbf{max\_blocks} & \textbf{Isotypic readout} & \textbf{Mean pooling} \\
\midrule
1   & 0.745 & 0.500 \\
2   & 0.775 & 0.500 \\
4   & 0.815 & 0.500 \\
8   & 0.890 & 0.500 \\
16  & 0.905 & 0.500 \\
32  & 0.910 & 0.500 \\
\bottomrule
\end{tabular}
\end{table}


\end{document}